%% file: main.tex
\documentclass[12pt]{spieman}  
\usepackage{amsmath,amsfonts,amssymb}
\usepackage{graphicx}
\usepackage{tocloft}
\usepackage{lineno}
\usepackage{enumitem}

\usepackage{mathtools}
\usepackage{hyperref}
\usepackage{cleveref}
\crefname{prop}{Proposition}{Propositions}
\Crefname{prop}{Proposition}{Propositions}

\usepackage{amsthm}
\usepackage{thmtools}
\usepackage{thm-restate}
\declaretheoremstyle[
  headfont=\normalsize\bfseries,
  bodyfont=\normalsize,
  notefont=\normalsize\itshape,
  notebraces={(}{)},
]{propstyle}

\usepackage{algorithm}
\usepackage[noend]{algpseudocode}
\usepackage{setspace}
\usepackage{paracol}
\usepackage{subcaption}
\usepackage{xcolor}
\usepackage{bm}

\DeclareMathOperator*{\argmin}{argmin\,}
\DeclareMathOperator*{\argmax}{argmax\,}
\newcommand\norm[1]{\lVert#1\rVert}
\def\x{{\mathbf x}}
\def\y{{\mathbf y}}
\def\s{{\mathbf s}}
\def\I{{\mathbf I}}
\def\f{{\mathbf f}}
\def\w{{\mathbf w}}
\def\cN{{\cal N}}
\def\cA{{\cal A}}
\def\cH{{\cal H}}
\def\E{{\mathbb E}}

\def\C{{\mathbb C}}
\def\P{{\mathbf P}}

\title{Learning Optimized Sampling Patterns for Accelerated MRI Reconstruction via Diffusion Models}

\author[a]{Sriram Ravula}
\author[a]{Brett Levac}
\author[a,b,*]{Yamin Arefeen}
\author[c]{Ajil Jalal}
\author[c]{Alexandros G.\ Dimakis}
\author[a,d,e]{Jonathan I.\ Tamir}
\affil[a]{Chandra Family Department of Electrical and Computer Engineering, The University of Texas at Austin, Texas, USA}
        \affil[b]{Imaging Physics, MD Anderson Cancer Center, Houston, Texas, USA}
\affil[c]{Electrical Engineering and Computer Sciences, University of California, Berkeley, California, USA}
\affil[d]{Oden Institute for Computational Engineering and Science, The University of Texas at Austin, Texas, USA}
\affil[e]{Department of Diagnostic Medicine, Dell Medical School, Austin, Texas, USA}

\cftpagenumbersoff{figure}
\cftpagenumbersoff{table} 
\begin{document} 
\maketitle

\begin{abstract}
\newline \textbf{Purpose:} Magnetic resonance imaging (MRI) is a powerful medical imaging modality, but long acquisition times limit throughput, patient comfort, and clinical accessibility. Diffusion-based generative models serve as strong image priors for reducing scan-time with accelerated MRI reconstruction and offer robustness across variations in the acquisition model. However, most existing diffusion-based approaches do not exploit the unique ability in MRI to jointly design both the sampling pattern and the reconstruction method. While prior learning-based approaches have optimized sampling patterns for end-to-end unrolled networks, analogous methods for diffusion-based reconstruction have not been established due to the computational burden of posterior sampling. In this work, we propose a method to optimize k-space sampling patterns for accelerated multi-coil MRI reconstruction using diffusion models as priors.
\newline \textbf{Approach:} We introduce a training objective based on a single-step posterior mean estimate that avoids backpropagation through an expensive iterative reconstruction process. Then we present a greedy strategy for learning Cartesian sampling patterns that selects informative k-space locations using gradient information from a pre-trained diffusion model while enforcing spatial diversity among samples. 
\newline \textbf{Results:} Experimental results across multiple anatomies and acceleration factors demonstrate that diffusion models using the optimized sampling patterns achieve higher-quality reconstructions in comparison to using fixed and learned baseline patterns.
\newline \textbf{Conclusions:} This work proposes a method to learn sampling patterns for accelerating MRI with diffusion-based generative models and posterior sampling reconstruction.
\end{abstract}


{\noindent \footnotesize\textbf{*}Corresponding author: Yamin Arefeen,  \linkable{yamin.arefeen@austin.utexas.edu} }

\newcommand{\fix}{\marginpar{FIX}}
\newcommand{\new}{\marginpar{NEW}}

\begin{spacing}{2}   

\section{Introduction}
\singlespacing
Magnetic Resonance Imaging (MRI) is a powerful and flexible medical imaging modality, providing excellent soft-tissue contrast without ionizing radiation. However, long acquisition times remain a fundamental limitation, contributing to reduced clinical accessibility, elevated healthcare costs, and extensive patient discomfort. Accelerating MRI acquisitions by reconstructing diagnostic quality images from shorter, under-sampled acquisitions has therefore been a central objective of MRI research for several decades. Parallel imaging, a widely adopted strategy to accelerate MRI, exploits the spatial sensitivity variation of multi-coil receive arrays to reconstruct missing k-space samples \cite{smash,sense,grappa}. Compressed Sensing (CS)  has also been used to accelerate MRI beyond the Nyquist rate by sampling a pseudo-random subset of Fourier coefficients in k-space and imposing a sparse prior on the image in some transform domain \cite{cs,donoho_cs,lustig2007sparse}. Parallel imaging and CS can also be combined, but these two methods impose competing requirements on the sampling pattern: compressed sensing favors incoherent, pseudo-random sampling, while parallel imaging relies on structured sampling that preserves local k-space neighborhoods to enable linear predictability \cite{haldar}. 

Recently, deep learning has achieved state-of-the-art performance in accelerated MRI reconstruction compared to classical parallel imaging and sparsity-based methods. These methods include unrolled techniques \cite{aggarwal2018modl, deepjsense, hammernik}, as well as approaches leveraging generative models \cite{deepdecoder,robustmri,chulscore,uecker_diffusionMRI} (see Heckel et al.\cite{Heckel2024DeepLF} for a recent survey). Besides the reconstruction algorithm, the acquired measurements in MRI can also be chosen by the user. Researchers have proposed many deep learning techniques to optimize the sub-sampling pattern for accelerated MRI \cite{wei2022, sparkling, weiss2021pilot, Alkan_2024, bjork, JMODL_2020, zibetti}, leading to improved reconstruction quality over hand-crafted or heuristic patterns. However, these approaches typically optimize a reconstruction network along with the sampling pattern in a joint fashion \cite{UNet_LOUPE}. This joint learning scheme is inflexible, requiring a new pattern and reconstruction network to be trained for each new set of imaging conditions. In addition, these approaches require the gradient of the sampling operator to be calculated through the full reconstruction process, making it unclear how to extend them to methods like CS with generative models and diffusion posterior sampling \cite{csgm, robustmri}.

In this work, we propose a novel method to optimize the acquired samples for MRI reconstruction using posterior sampling with generative diffusion models \cite{robustmri,chulscore,uecker_diffusionMRI,songmri}, which has been shown to be robust to changes in imaging anatomy, acceleration factor, and sampling patterns, without requiring retraining. One of our main challenges is that posterior sampling can require tens or hundreds of iterations, making it difficult to unroll or invert the reconstruction algorithm to propagate training gradients for updating our sampling pattern. We therefore present a method to tackle the issue of obtaining gradients from posterior sampling and introduce an approach for learning sampling patterns using these gradients. Specifically, we prove a theoretical extension to Tweedie's formula that incorporates both the learned prior from the diffusion model and the MRI measurements, then use this result to propose a simple and effective training objective that allows us to update our sampling operator using gradients calculated with just a single reconstruction step, avoiding the computational and memory burden associated with backpropagating through a full posterior sampling process. We then propose to learn sampling patterns in an iterative greedy process by selecting candidate k-space samples based on their gradient values and distance away from existing samples. We apply our approach to multi-coil MRI with a Cartesian acquisition scheme and demonstrate that our optimized sampling patterns lead to high-quality reconstructions across multiple anatomies and acceleration factors. Fig. \ref{fig:brain_patterns} showcases optimized sampling patterns for diffusion based reconstructions enabled by our approach.
\begin{figure}[!t]
    \centering
    \includegraphics[width=0.99\linewidth]{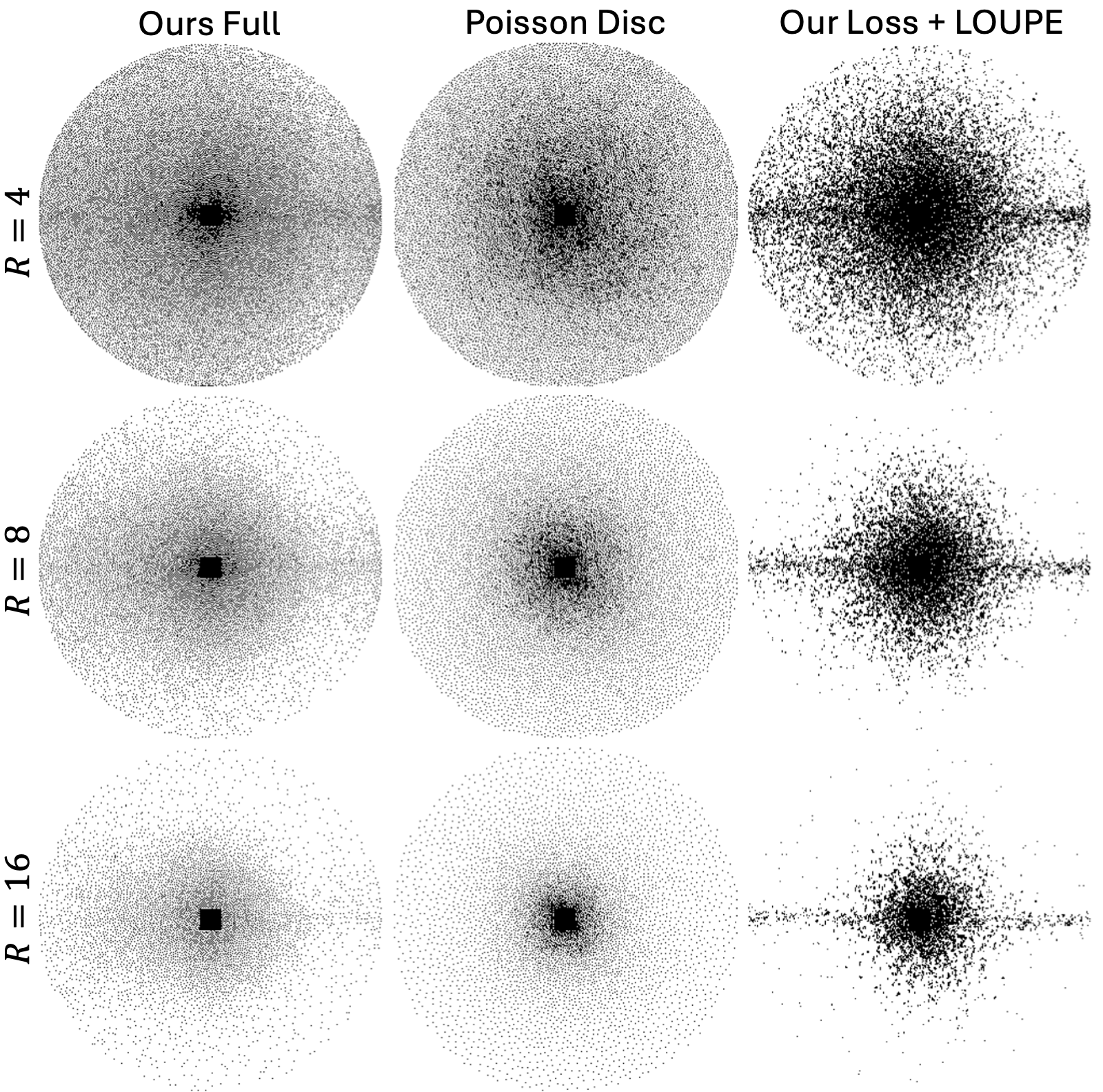}
    \caption{\textbf{Sampling patterns from different methods for T2 brain data.} Optimized sampling patterns from our full proposed method used for reconstructing slices at acceleration factors $R$ in $\{4, 8, 16\}$ with the diffusion-based methods. Poisson Disc sampling tends to distribute samples more uniformly at high acceleration rates, while LOUPE concentrates many samples near the center of k-space. The proposed method balances these effects by maintaining spacing between samples while preferentially allocating measurements to informative central regions.}
    \label{fig:brain_patterns}
\end{figure}

We summarize our contributions below:
\begin{itemize}
    \item We propose a novel training objective for selecting which samples to obtain for MRI reconstruction using diffusion-based generative models. Our method is based on a theoretical result which provides a closed-form solution for calculating the mean of the posterior distribution.
    \item We use pre-trained diffusion models as priors without requiring further tuning, making our method modular. 
    \item We train Cartesian sampling patterns with a greedy update by selecting informative samples based on their gradient values while ensuring that chosen samples are not clustered closely together. 
    \item Empirical results show that diffusion-based posterior sampling achieves higher reconstruction quality when applying our optimized sampling pattern in-comparison to baseline sampling patterns across a range of anatomy and acceleration rates.
\end{itemize}

\section{Related Work}

\subsection{Compressed Sensing with Generative Models} 

The theory of compressed sensing \cite{cs, donoho_cs} demonstrates that a signal can be accurately recovered from measurements taken below the Nyquist rate, given that the signal is sparse in an appropriately-chosen basis. The authors of CSGM \cite{csgm} extended this finding by showing that the range of a generative model is a more powerful prior than sparsity. Generative models have been used as priors for inverse problems in deblurring \cite{jaydeblurring}, phase retrieval \cite{phase_retrieval}, 1-bit compressed sensing \cite{1bit}, and many more. Various algorithms have been proposed to leverage generative models for these reconstruction tasks \cite{ajilicml21, ddrm, song2021scorebased}, including methods which use untrained models \cite{dip, deepdecoder} or models trained entirely on compressive measurements \cite{ambientgan, aali2024ambient}. See the survey by Ongie et al.\cite{imaging_survey} for a taxonomy of image reconstruction algorithms and applications. 

\subsection{MRI Sampling Pattern Selection}

Previous works for learning optimal sampling patterns for accelerated MRI, broadly speaking, parameterize the sampling pattern as either Cartesian or Non-Cartesian.

Cartesian techniques learn to sample on a uniformly spaced, discrete grid of points in k-space, but this parameterization presents a challenging combinatorial optimization problem with an enormous search space: for a fixed acceleration factor $R=\frac{m}{n}$, there are $\binom{n}{m}$ possible patterns to choose from. The authors of LOUPE \cite{UNet_LOUPE, multicoil_loupe} relax the problem by parameterizing continuous logits underlying independent Bernouli random variables at each k-space location. They train the logits jointly with a reconstruction network in an end-to-end fashion using gradient descent. Alternative approaches use subset selection-based algorithms and avoid differentiating w.r.t. the pattern, selecting sampling points based on their contribution to the error \cite{zibetti, zibetti2}, or variance \cite{wang2023adaptive} if fully-sampled k-space is not available. 

Non-Cartesian methods exploit the fact that MRI is not restricted to grid-based sampling schemes. These methods use a version of the MRI forward operator based on the non-uniform FFT (NuFFT), which is evaluated at a given set of coordinates $\phi$ and is continuous and differentiable w.r.t. these coordinates. Some methods optimize directly over the coordinates \cite{Alkan_2024}, some parameterize sampling trajectories over the coordinates using curves or lines \cite{bjork, sparkling, JMODL_2020}, and others use a hybrid approach by first learning coordinates and then optimizing a trajectory through those coordinates \cite{weiss2021pilot}. Additional work has been done to enforce physical constraints on the learned sampling pattern to avoid undesirable effects from the imaging process, such as peripheral nerve stimulation (PNS) \cite{SNOPY}.

Our work focuses on learning Cartesian sampling patterns since they are more widely used in practice than non-Cartesian patterns and offer a straightforward imaging model that leverages highly-optimized FFT implementations \cite{wright2014non}. Unlike the majority of previous methods that learn Cartesian patterns by optimizing a sampling density \cite{UNet_LOUPE, sherry}, our approach produces a single sampling pattern instance that accounts for redundancy and interdependence between k-space locations. The majority of prior approaches also jointly optimize the sampling pattern with a reconstruction network, leading to a ``single-use'' framework that is brittle to shifts in the sampling scheme or image distribution at inference. In contrast, our method leverages a powerful pre-trained diffusion model with frozen weights that is not tied to any sampling scheme, can be re-used to learn multiple patterns, and is robust to distribution shifts at inference \cite{ajilicml21}.

Our method is most similar to the method presented in \cite{wang2023adaptive}, which uses a score-based model to learn an adaptive Cartesian pattern by sampling multiple reconstructions at each iteration using posterior sampling and greedily adding the k-space points with the highest variance over the reconstructions. Our method differs in that we focus on the non-adaptive setting where fully-sampled reference data are available offline during training. Further, we propose an objective that offers a simple alternative to full posterior sampling during training, considerably reducing computational costs. Note, portions of this work served as material for Ravula's Dissertation \cite{sriramThesis}.

\section{Background}

\subsection{MRI Reconstruction}
\label{sec:MRI}

A common approach for accelerating MRI scans is to acquire fewer measurements, which leads to an ill-posed inverse problem that must be solved to reconstruct the image. The measurement process for multi-coil MRI can be expressed as follows:
\begin{equation}
\label{eqn:MRI_forw}
    \mathbf{y}_i = \mathbf{PFS}_i \x_0 + \bm{\epsilon}_i,
\end{equation}
where $\mathbf{y}_i\in \C^m$ are the measurements in the spatial frequency domain (or \emph{k-space}) for the $i^{th}$ coil, $\x_0\in \C^n$ is the image of interest, $\mathbf{S}_i\in \C^{n\times n}$ is the coil sensitivity map for the $i^{th}$ coil ($c$ coils in total), $\mathbf{F}\in \C^{n \times n}$ is the Fourier transform matrix, $\mathbf{P} \in \C^{m \times n}$ is a binary sub-sampling operator with rows taken from the $n \times n$ identity matrix, and $\bm{\epsilon}_i \in \C^{m}$ is i.i.d Gaussian noise. We define the \emph{acceleration factor} as the under-sampling ratio given by $R \coloneqq m / n$. We note that even if there are many coils and $m \times c \geq n$, the inverse problem may still be ill-posed due to spatial correlations in the coil sensitivity maps. We denote $\mathbf y$ to be the multi-coil k-space measurements and $\cA$ the total multi-coil MRI measurement model. 

\subsection{Diffusion-Based Generative Models}
\label{sec:diffusion}

Diffusion models learn to generate signals by reversing a corruption process, specifically by removing additive Gaussian noise, where the noise magnitude is controlled by the variance $\sigma^2_t$ at each time step $t \in [0, T]$. Typically, $\sigma^2_0 = 0$ corresponds to clean signals, while $\sigma^2_T$ is large enough that the noisy signals become indistinguishable from pure Gaussian noise. Song et al. \cite{song2021scorebased} unify Score-Based Models (SBMs) \cite{song2019generative} and Denoising Diffusion Probabilistic Models (DDPMs) \cite{ho2020denoising} under continuous-time Stochastic Differential Equations (SDEs), interpreting SBMs as reversing \emph{Variance Exploding} (VE) SDEs and DDPMs as reversing \emph{Variance Preserving} (VP) SDEs. We focus on the VE-SDE framework.

The noisy signal at time $t$ in the forward diffusion process is given by $\x_t$. The diffusion process is modeled as the solution to an It\^o SDE of the form
\begin{equation}
\label{eq:forward_ve}
    d\x = \f(\x, t)dt + g(t)d\w.
\end{equation}
Here, $\w$ is the standard Wiener process, and in the VE case, $\f(\x, t) = \mathbf{0}$ and $g(t) = \sqrt{\frac{d\sigma_t^2}{dt}}$. The variance $\sigma^2_t$ is a monotonically increasing function that defines the distribution of the diffused signal, with the property that at time $t=0$ we recover the data distribution: $\x_0 \sim p_{0} = p_{data}$. 

The goal of diffusion-based models is to start from samples $\x_T \sim p_T$ consisting of pure Gaussian noise and reverse the forward diffusion given by Eq.~\eqref{eq:forward_ve} to arrive at samples $\x_0 \sim p_0$. Conveniently, the reverse of the forward SDE is also an SDE \cite{ANDERSON1982313}, with the form:
\begin{equation}
\label{eq:reverse}
    d\x = [\f(\x, t) - g(t)^2 \nabla_{\x_t} \log p_t(\x_t)]dt + g(t)d\Bar{\w},
\end{equation}
where $dt$ is now a negative time step and $\Bar{\w}$ is the standard Wiener process when time flows backward. The reverse SDE depends on the \emph{score function} $\nabla_{\x_t} \log p_t(\x_t)$ of the marginal distribution at time $t$. 

A result from Vincent \cite{vincent2011connection} states that we can learn the \emph{unconditional} score function in Eq.~\eqref{eq:reverse} through denoising score matching (DSM). DSM only requires the score of the \emph{conditional} noise distribution at time $t$, which is known in closed form for additive Gaussian noise. The training objective is
\begin{align*}
\label{eq:dsm}
    \mathbf{\theta}^* = \argmin_{\mathbf{\theta}} \E_{t \sim U[0, T]} \big[ \lambda_t \E_{(\x_0, \x_t) \sim p_0(\x_0) p_t(\x_t | \x_0)} \norm{\s_{\mathbf{\theta}}(\x_t, t) \\ - \nabla_{\x_t} \log p_t(\x_t | \x_0)}^2 \big],  
\end{align*}
where $\lambda_t$ is a time-dependent positive weighting function and $\s_{\mathbf{\theta}}(\x_t, t)$ is the \emph{score network}. Given enough data and model capacity, the score network learns to approximate the unconditional score function: $\s_{\mathbf{\theta}}(\x_t, t) \simeq \nabla_{\x_t} \log p_t(\x_t)$. Subsequently, sampling with the score network is a matter of substituting $\s_{\mathbf{\theta}}(\x_t, t)$ for the score function in Eq.~\eqref{eq:reverse}, then solving the reverse SDE.   

\subsection{Accelerated MRI Reconstruction with Diffusion Models}
\label{sec:inverse}

Bayes' rule gives the relationship between the \emph{prior} distribution $p(\x)$ and the \emph{posterior} distribution $p(\x|\y)$ as $p(\x|\y) = p(\x)p(\y|\x) / p(\y)$. If we take the gradient w.r.t. $\x$ of the log of both sides, we find that
\begin{equation}
\label{eq:posterior_score}
    \nabla_{\x} \log p(\x|\y) = \nabla_{\x} \log p(\x) + \nabla_{\x} \log p(\y|\x). 
\end{equation}
Returning to the SDE framework, we can sample from the posterior distribution by using the decomposed posterior score from Eq.~\eqref{eq:posterior_score} in place of the score in Eq.~\eqref{eq:reverse}:
\begin{equation}
\label{eq:reverse_posterior}
    \begin{aligned}
        d\x = [\f(\x, t) - g(t)^2 &(\nabla_{\x_t} \log p_t(\x_t) \\&+ \nabla_{\x_t} \log p_t(\y|\x_t))]dt + g(t)d\Bar{\w}.
    \end{aligned}
\end{equation}
We need to know the score of the prior, $\nabla_{\x_t} \log p_t(\x_t)$, and the likelihood, $\nabla_{\x_t} \log p_t(\y|\x_t)$. To find the score of the prior distribution, we can train a diffusion model as described in Sec.~\ref{sec:diffusion}. However, the time-dependent likelihood $p_t(\y|\x_t)$ is not easy to obtain. In graphical terms, $\x_0 \rightarrow \y$ and $\x_0 \rightarrow \x_t$, with no other dependency between the measurements $\y$ and the noisy signal $\x_t$.

The authors of Diffusion Posterior Sampling (DPS) \cite{chung2023diffusion} propose to approximate $\nabla_{\x_t} \log p_t(\y|\x_t)$ by exploiting a result from Tweedie \cite{efron2011tweedie}. Tweedie's formula gives a closed-form expression for the mean of $p_0(\x_0|\x_t)$ when $p_t(\x_t|\x_0)$ belongs to an exponential family distribution. For the case of VE diffusions, we have that $p_t(\x_t|\x_0) = \cN(\x_t; \x_0, \sigma_t^2 \I)$, with the posterior mean given by Tweedie as
\begin{equation}
\label{eq:tweedie}
    \hat{\x}_0(\x_t) \coloneqq \E_{\x_t \sim p_t(\x_t|\x_0)}[\x_0 | \x_t] = \x_t + \sigma_t^2 \nabla_{\x_t} \log p_t(\x_t). 
\end{equation}
This formula allows us to get a ``one-step'' denoised estimate of the clean signal $\x_0$ using only the noisy signal $\x_t$ and the score of the unconditional distribution, $\nabla_{\x_t} \log p_t(\x_t)$. We can replace the unconditional score in Eq.~\eqref{eq:tweedie} using our trained score network $\s_{\mathbf{\theta}}(\x_t, t)$ to approximate the posterior mean as
\begin{equation}
\label{eq:tweedie_score}
    \hat{\x}_0(\x_t) = \E_{\x_t \sim p_t(\x_t|\x_0)}[\x_0 | \x_t] \simeq \x_t + \sigma_t^2 \s_{\mathbf{\theta}}(\x_t, t).
\end{equation}
DPS approximates the time-dependent score of the likelihood as
\begin{equation}
\label{eq:dps}
    \nabla_{\x_t} \log p_t(\y|\x_t) \simeq \nabla_{\x_t} \log p_t(\y|\hat{\x}_0(\x_t)). 
\end{equation}
For the setting of Gaussian measurement noise, as is the case in MRI, $\y \sim \cN(\y; \cA(\x_0), \sigma^2_{\y}\I)$, the approximation in Eq.~\eqref{eq:dps} gives us
\begin{equation}
\label{eq:likelihood_grad}
    \nabla_{\x_t} \log p_t(\y|\x_t) \simeq - \frac{1}{2\sigma_{\y}^2} \nabla_{\x_t} \norm{\cA(\hat{\x}_0(\x_t)) - \y}^2,
\end{equation}
where $\hat{\x}_0(\x_t)$ is calculated as in Eq.~\eqref{eq:tweedie_score}. Finally, using the approximation from Eq.~\eqref{eq:likelihood_grad}, we can sample from the posterior using the reverse SDE from Eq.~\eqref{eq:reverse_posterior}. In practice, calculating the gradient in Eq.~\eqref{eq:likelihood_grad} requires backpropagating through the score network with respect to the noisy input $\x_t$.

\subsection{Sampling Pattern Optimization}
\label{sec:sampling_optimization_background}
Sampling pattern optimization aims to learn a sampling pattern $\P$ that is used to produce sub-sampled MRI measurements $\mathbf{y}_i = \mathbf{PFS}_i \x_0 + \mathbf{\epsilon}$ with the goal of minimizing reconstruction error between the true signal $\x_0$ and the estimate $\tilde{\x}_0$ produced by some reconstruction algorithm. Methods typically assume access to a training dataset of fully-sampled k-space measurements. The optimization problem can be written as follows:
\begin{equation}
\label{eq:meas_learning}
    \P^* = \argmin_{\P} \E_{\x_0 \sim p_0(\x_0),\: \tilde{\x}_0 \sim p_0(\x_0 | \y)} \norm{\x_0 - \tilde{\x}_0(\P)}^2,
\end{equation}
where we write the reconstruction as $\tilde{\x}_0(\P)$ to make it explicitly a function of the sampling pattern.
\section{Methods}
\label{sec:methods}

Our goal is to learn a sampling pattern $\P$ that minimizes reconstruction error between the true signal $\x_0$ and the estimate $\tilde{\x}_0$ produced by posterior sampling using a diffusion-based generative model. In our setting, we are given a fixed, desired acceleration factor $R$ for the learned pattern. 
We also assume access to a training dataset $\{(\y^j, \mathbf{S}^j)\}_{j=1}^{N_{train}}$ of fully-sampled k-space measurements $\y^j$ and their corresponding coil sensitivity maps $\mathbf{S}^j$ (where, for clarity, $\mathbf{S}^j$ encapsulates $c$ total coil maps $\mathbf{S}_1^j, \dots, \mathbf{S}_c^j$), noting that we can derive reference images $\x_0^j$ through the pseudo-inverse of the sampling operator using $\y^j$ and $\mathbf{S}^j$. 
In addition, we are given a diffusion model $\s_{\mathbf{\theta}}(\x_t, t)$ pre-trained on the distribution of reference images $p_0(\x_0)$. 

In practice, there are two challenges to address when optimizing sampling patterns by solving Eq.~\eqref{eq:meas_learning} for generative diffusion based reconstruction algorithms. First, sampling from the posterior to get $\tilde{\x}_0$ involves an iterative application of a differential equation solver. This makes differentiating with respect to $\P$ challenging, as na\"ively backpropagating through the sampling procedure is infeasible due to memory constraints. Second, it is unclear how to best parameterize and optimize the discrete, binary sampling pattern. A common solution is to relax the problem by learning a distribution over sampling patterns instead of the pattern itself. However, drawing sampling patterns from the learned distribution usually involves spatially independent selection of k-space points based on their probability values, often leading to the acquisition of samples containing redundant information in the case of multi-coil imaging.

In this section, we outline our proposed solutions to the above two problems. First, we derive an extension to Tweedie's formula specifically for the case of posterior sampling. Our result leads us to a simpler form of the problem in Eq.~\eqref{eq:meas_learning} that makes gradient-based methods computationally tractable. Second, we detail a novel approach for learning Cartesian sampling patterns, visualized in Fig. \ref{fig:overview}. Our method
involves an iterative greedy process of selecting the k-space locations that are most informative based on the value of their loss gradients as well as proximity to previously-selected k-space locations. Taken together, these two innovations form a flexible and powerful new method to optimize MRI sample selection for reconstruction using diffusion-based models. 
\begin{figure*}[!t]
    \centering
    \includegraphics[width=0.8\linewidth]{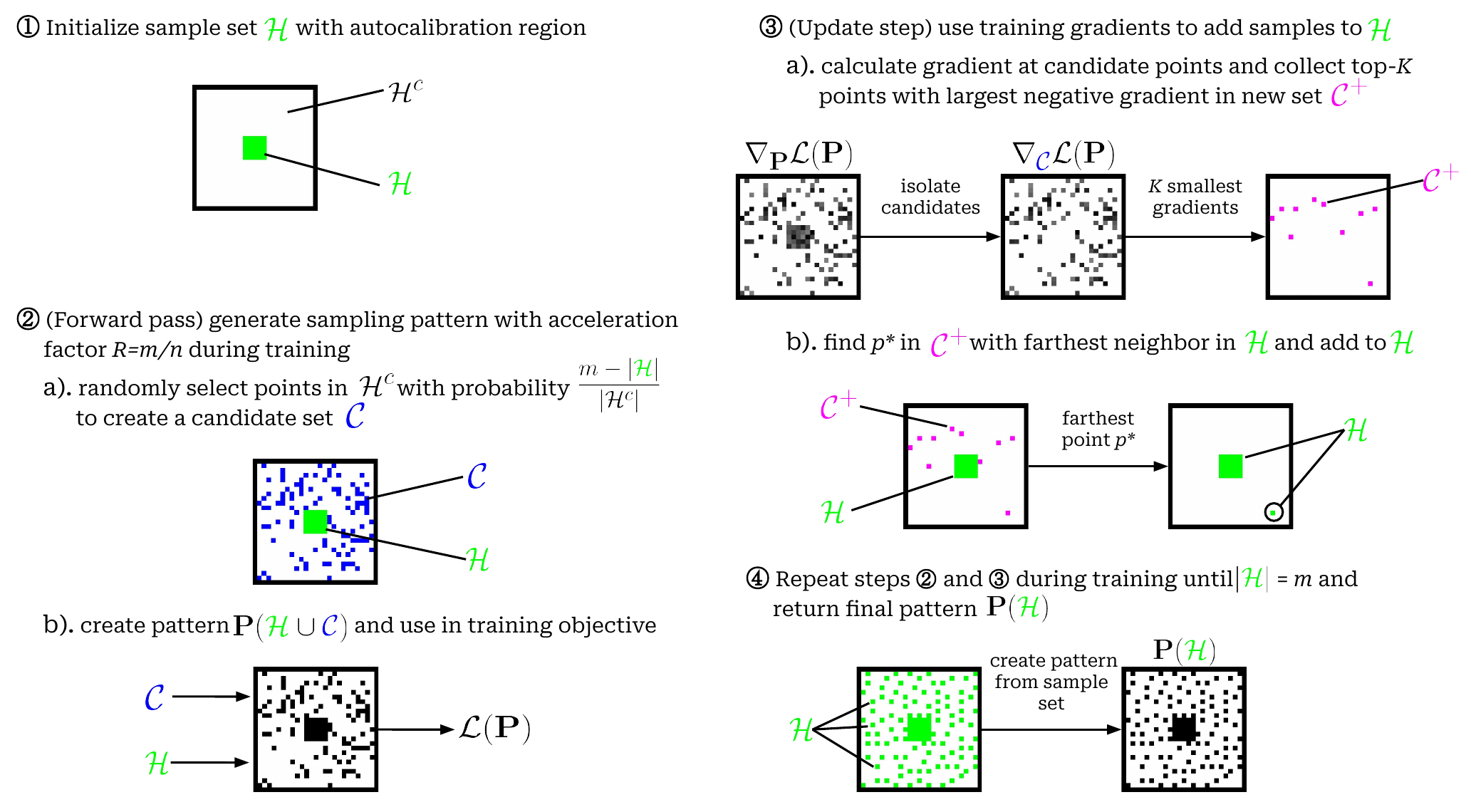}
    \caption{\textbf{Overview of our method for learning sampling patterns.} Starting with the autocalibration region, we greedily grow a set $\mathcal{H}$ (shown in \textcolor{green}{green} in the figure) of k-space sampling locations on the Cartesian grid. At each training iteration, we generate a pattern with proper acceleration $R=\frac{m}{n}$ using the previously selected points from $\mathcal{H}$ along with randomly chosen \textit{candidate points} in k-space that have not yet been added to $\mathcal{H}$. Then, we use the generated pattern in our diffusion training objective to get a loss term and calculate the gradient of the loss at each candidate sampling location. Next, we isolate the $K$ candidate points with the smallest (i.e., largest negative) gradient values, select the point that is farthest away (in terms of Euclidean distance) from its nearest neighbor in $\mathcal{H}$, and finally add that point to $\mathcal{H}$. Our training proceeds, adding one k-space location at a time, until $|\mathcal{H}| = m$ and our desired acceleration is reached.}
    \label{fig:overview}
\end{figure*}

\subsection{Loss Formulation}
\label{sec:loss}

In this section, we present a extension to Tweedie's formula that we then use to formulate a simpler alternative to the problem in Eq.~\eqref{eq:meas_learning}. For the case of VE SDEs, recall that the noisy signal distribution at time $t$ is given by $p_t(\x_t|\x_0) = \cN(\x_t; \x_0, \sigma_t^2 \I)$, allowing us to use Eq.~\eqref{eq:tweedie} to get a one-step approximation of the denoised posterior mean $\E[\x_0 | \x_t]$. We now extend Tweedie's formula to include measurements $\y \sim p(\y | \x_0)$.     

\begin{restatable}[Tweedie's formula with additional measurements]{prop}{tweedie}
\label{prop:tweedie}
    Let $\x_0 \sim p_0(\x_0)$ be an unknown signal, $\x_t \sim p_t(\x_t | \x_0) = \cN(\x_t; \x_0, \sigma^2_t \I)$ a version of $\x_0$ corrupted by additive Gaussian noise, and $\y \sim p(\y | \x_0)$ some additional measurements of $\x_0$. Furthermore, let $\x_t$ and $\y$ be conditionally independent given $\x_0$: $p_t(\x_t | \x_0, \y) = p_t(\x_t | \x_0)$. Finally, assume that $p_t(\x_t | \y)$ is supported everywhere. Then, the posterior mean of $\x_0$ conditioned on $\x_t$ and $\y$ is given by  
    \begin{equation}
    \label{eq:posterior_tweedie}
        \E[\x_0 | \x_t, \y] = \x_t + \sigma^2_t \nabla_{\x_t} \log p_t(\x_t | \y). 
    \end{equation}
\end{restatable}
\sloppy
We give the proof in Appendix A. Recalling that $\nabla_{\x_t} \log p_t(\x_t|\y) = \nabla_{\x_t} \log p_t(\x_t) + \nabla_{\x_t} \log p_t(\y|\x_t)$, our result essentially allows us to leverage the score of the likelihood \emph{in addition to} the prior to obtain a finer estimate of $\x_0$ than using the prior alone. The assumption of conditional independence of $\x_t$ and $\y$ given $\x_0$ is satisfied in the inverse problem setting, as we have that $\x_0 \rightarrow \y$ and $\x_0 \rightarrow \x_t$ with no other dependencies.

In practice, we only have access to a score network and a surrogate form of the score of the time-dependent likelihood from Eq.~\eqref{eq:dps}. Therefore, a tractable approximation of the expectation in Eq.~\eqref{eq:posterior_tweedie} is
\begin{equation}
\label{eq:posterior_tweedie_appx}
\begin{aligned}
    \E[\x_0 | \x_t, \y] &\simeq \x_t + \sigma^2_t [\s_{\mathbf{\theta}}(\x_t, t) + \nabla_{\x_t} \log p_t(\y|\hat{\x}_0(\x_t))] \\
    & = \hat{\x}_0(\x_t) + \sigma_t^2 \nabla_{\x_t} \log p_t(\y|\hat{\x}_0(\x_t)),
\end{aligned}
\end{equation}
where $\hat{\x}_0(\x_t)$ is given by Tweedie's formula as in Eq.~\eqref{eq:tweedie_score}. 

For the multi-coil MRI setting with coil maps $\mathbf{S}_1,\dots,\mathbf{S}_c$, we assume that $p(\y_i|\x_0) = \cN(\y_i; \mathbf{PFS}_i\x_0, (\sigma^2_{y}/2)\I)$, and use $\y$ to denote the collection of measurements $\y_i$ from all $c$ coils. Further, we take the gradient of the likelihood term in Eq.~\eqref{eq:posterior_tweedie_appx} with respect to the denoised estimate $\hat{\x}_0$ as opposed to the noisy $\x_t$, which would require backpropagating through the diffusion model. This change considerably reduces memory and computational costs during training, and we find it works well in practice. Our final formula for calculating the posterior mean is
\begin{equation}
\label{eq:posterior_tweedie_final}
    \E[\x_0 | \x_t, \y] \simeq \hat{\x}_0(\x_t) - \gamma \nabla_{\hat{\x}_0} \sum_{i=1}^{c} \norm{\mathbf{PFS}_i\hat{\x}_0(\x_t) - \y_i}^2,
\end{equation}
where $\gamma$ is a likelihood step size hyperparameter. Using this result, we finally present our training objective:
\begin{equation}
\label{eq:our_method}
\begin{gathered}
    \P^* = \argmin_{\P} \E_{\x_0 \sim p_0(\x_0)} \norm{\x_0 - \tilde{\x}_0(\P)}^2, \\ \tilde{\x}_0(\P) = \E_{\: \y \sim p(\y | \x_0),\: \x_t \sim p_t(\x_t|\x_0) , \: t \sim U[0, T]}[\x_0 | \x_t, \y].
\end{gathered}
\end{equation}

Starting with a noisy signal at time $t$, samplers for diffusion models generally follow the approach of: get the denoised estimate at time $t$ using the score network, then add back the appropriate amount of noise for time $t-1$ and repeat. Therefore, by learning forward operators that produce good one-step posterior denoised estimates, our method facilitates good final reconstructions. We are able to use powerful generative models ``off-the-shelf'' without further training.

\subsection{Learning Sampling Patterns}
\label{sec:pattern}

We propose a novel method for learning Cartesian sampling patterns for multi-coil accelerated MRI. Our method proceeds in iterations, greedily growing a set of k-space samples by considering both how informative a potential new sample is for reducing training error and how redundant that sample is with samples that are already included in the set. We learn a single sampling pattern instance for a given set of multi-coil, fully-sampled k-space training data and a desired acceleration factor. Our method does not require end-to-end or joint training with a reconstruction method, and instead flexibly interfaces with a powerful, pre-trained diffusion model with frozen weights. 

First, we describe how we represent learned patterns and define our problem more clearly. Let $\Omega$ be the set of all possible k-space sampling locations (or \textit{points}). In Cartesian sampling, $\Omega$ is a finite grid of size $|\Omega|=n$ consisting of equally spaced, discrete points $p \in \Omega$. Each point $p$ is associated with a unique coordinate vector $\mathbf{v}(p) = [x \:\: y]^T$ that indicates the location of the point on the grid, with the origin $[0 \:\: 0]^T$ at the DC component and each grid dimension normalized to the range $[-1,1]$. We can write a sampling pattern as a function $\P(\cH)$, where $\cH \subseteq \Omega$ contains the sampled k-space points. We consider the problem of learning an optimal set $\cH$ of size $|\cH|=m$, which is equivalent to learning an optimal pattern $\P$ for a given acceleration factor $R = m/n$.

We learn the sample set $\cH$ in an iterative manner by greedily adding new k-space points to the set until the desired acceleration is reached. Prior to the start of training, we initialize $\cH$ with the points in the autocalibration region, which are always sampled. At each training iteration, we first randomly select candidate k-space points from outside of $\cH$, then generate a sampling pattern that includes the points in $\cH$ as well as the new candidate points, compute the training loss using the pattern and our objective in Eq.~\eqref{eq:our_method}, and finally select one candidate point to add to $\cH$ based on our greedy update rule. Specifically, let $\cH^C \coloneqq \Omega / \cH$ denote the k-space locations that we have not added to our learned set. We select points $p \in \cH^C$ independently with probability $(m - |\cH|)/|\cH^C|$, forming a candidate set $\mathcal{C}$ of the selected points. Then, we generate a sampling pattern $\P(\cH \cup \mathcal{C})$ that merges the current set $\cH$ with the candidates in $\mathcal{C}$ and calculate training loss $\mathcal{L}(\P)$ following Eq.~\eqref{eq:our_method}. Note that since $\E[|\cH \cup \mathcal{C}|] = m$, our approach maintains the correct acceleration factor on average during training. To guide our selection of which candidate to add to $\cH$, we compute the gradient of the training loss with respect to each point in $\mathcal{C} = \{p_1, \dots, p_{|\mathcal{C}|}\}$, which we denote as $\nabla_{\mathcal{C}} \mathcal{L}(\P) \coloneqq [\nabla_{p_1} \mathcal{L}(\P) \: \dots \: \nabla_{p_{|\mathcal{C}|}} \mathcal{L}(\P)]^T$. Next, we detail how we choose a candidate to add at each iteration.         

Our key innovation is a greedy update that selects candidate k-space points to add to our learned set $\cH$ based on both the informativeness of the candidate point and its incoherence with existing sampling locations. First, we fix a parameter $K \in \mathbb{Z}^+$ and determine the $K$ candidate sampling points $p \in \mathcal{C}$ with the smallest gradient values $\nabla_{p} \mathcal{L}(\P)$ (i.e. the largest negative gradient values). We collect these points in a new \textit{top-K} candidate set $\mathcal{C}^+$. Next, we solve the problem
\begin{equation}
\label{eq:furthest_neighbor}
p^* = \argmax_{p \in \mathcal{C}^+} \min_{p' \in \cH} ||\mathbf{v}(p) - \mathbf{v}(p')||^2.
\end{equation}
In other words, we find the candidate sampling point whose nearest neighbor in the set of existing, learned points $\cH$ is the furthest away. We finally add $p^*$ to $\cH$ and complete the training iteration. If there are multiple points with the same optimal value in Eq.~\eqref{eq:furthest_neighbor}, we randomly choose one point. Training proceeds by growing the size of our sampling set by one point at each iteration until we reach the desired acceleration when $|\cH| = m$. We present a visual overview of our learning procedure in Fig.~\ref{fig:overview}. 

By filtering the candidate sampling points based on their contribution to decreasing the training loss, we isolate the samples that are most informative about the structure of the MRI data. We then pick spatially isolated sampling locations, reducing the amount of redundant information captured in the sampling pattern and encouraging exploration of sparsely-sampled k-space regions. Further, by sequentially choosing sampling locations to remain ``on'', we force updates in subsequent iterations to consider these existing points, unlike methods that learn distributions over sampling patterns.

\section{Experimental Details}
\label{sec:experiments}

\subsection{Dataset}

We experiment on multi-coil brain and knee scan data from fastMRI \cite{zbontar2018fastmri}. For brains, we use T2-weighted volumes collected at a field strength of 1.5 T. For knees, we use Proton Density (PD) volumes without and with fat suppression (PDFS), collected at a field strength of 3 T. All MRI data are initially stored as complex-valued multi-coil k-space measurements. 

We pre-processed the raw data by reducing the FOV in the read out direction by a factor of two and performing noise whitening. We use the processed data to calculate coil sensitivity maps using ESPIRiT \cite{uecker2014espirit} as well as minimum variance unbiased estimator (MVUE) images to serve as our fully sampled reference during training and testing. We give full details about data preparation in Appendix B.

After pre-processing the raw fastMRI T2 brain data, we obtained 14162, 3110, and 280 samples for the training, validation, and test sets respectively. Similarly, for the PD knee data we had 6080, 864, and 281 samples for the training, validation, and test sets respectively. Finally, the training, validation, and test splits of the PDFS data were 6211, 880, and 270 samples. In the diffusion experiments, the training set was used to train the diffusion model, then the validation set was used to learn the sampling patterns with the pre-trained diffusion model. For the end-to-end/unrolled reconstruction network baselines, the training and validation sets were combined during training. We evaluate all methods on the test sets for their respective anatomies.

\subsection{Diffusion Models}

\subsubsection{Training}

We train diffusion models following the EDM \cite{Karras2022edm} framework and pre-conditioning. We train one model using the T2 brain data and one using the combined PD and PDFS knee data. The brain model has $13.5M$ parameters and is trained with a batch size of 54, while the knee model has $14.5M$ parameters and is trained with batch size 27. Both models were trained for $10M$ samples using the Adam \cite{kingma2015adam} optimizer with a learning rate of $10^{-4}$, an exponential moving average (EMA) half-life of $50K$ samples, and random horizontal flips for data augmentation. To make the complex-valued data compatible with real-valued network weights, we represent the data as two-channel real-valued images.

\subsubsection{Reconstruction Algorithm}

We use DPS \cite{chung2023diffusion} as our posterior sampling algorithm for reconstructing MRI images. We present our posterior sampler in Appendix C. Following the implementation of DPS, we set the log-likelihood step size parameter as $\rho_{dps} = \rho / \norm{\y - \cA(\hat{\x}_0(\x_t))}$, where $\rho$ is a tuneable hyperparameter. For PDFS knees and T2 brains, we set $\rho = 2.5$, for PD knees, we set $\rho = 5$, and for the combined knees, we set $\rho = 3$. We use 200 sampling steps with stochasticity parameter $S_{\mathrm{churn}} = 40$ and the default EDM noise schedule. For each set of MRI measurements, we draw five reconstructions from the posterior and take the mean of the reconstructions as our final output.   

\subsection{Learning Sampling Patterns}

For a given acceleration $R$, we train patterns with a corresponding batch size of $R$ samples in all experiments. To arrive at this value, we performed a search over batch sizes in $\{1, \frac{R}{4}, \frac{R}{2}, R, 2R\}$. We choose batch size as a multiple of $R$ since (i) patterns with lower acceleration factors are less sensitive to variance in the training gradients since they have a larger sampling budget (and vice-versa for higher accelerations), and (ii) this ensures that patterns learned with different acceleration factors see (roughly, after accounting for a fixed autocalibration region) the same number of samples during training. Similarly, we choose the top-$K$ parameter from $\{1, \frac{R}{4}, \frac{R}{2}, R, 2R\}$. For all experiments, we initialize $K=1$ and linearly increase $K$ over training to its final value, which we find improves results over fixing a single value of $K$. We choose $\gamma=1$ for the likelihood step size parameter in Eq.~\eqref{eq:posterior_tweedie_final} during training for all experiments.

\subsection{Overview of Compared Approaches}

Our experiments are designed to study how the proposed objective and sampling pattern optimization strategy affects diffusion-based MRI reconstruction. To this end, we first present comparisons of diffusion-based reconstruction with fixed or learned sampling patterns, which form the primary focus of this work. Then, our quantitative results also include end-to-end learned reconstruction. As the goal of this work is not to evaluate the performance of diffusion versus end-to-end reconstruction methods, these end-to-end methods are not for direct comparisons but rather included as reference points for contextualization. All sampling patterns (i) include a fixed $20\times20$ fully sampled autocalibration region at the center of k-space, and (ii) are constrained to a circular region inscribed in the rectangular k-space grid (cut-corner patterns). We more explicitly detail our primary diffusion comparisons below:

\noindent \textbf{Diffusion-based reconstruction with variable density Poisson Disc patterns} (\textit{Diffusion+Poisson}): As a baseline for diffusion-based methods without sampling optimization, we reconstruct images using a pre-trained diffusion model with non-learned variable-density Poisson Disc sampling generated with sigpy \cite{ong2019sigpy}. 

\noindent \textbf{Diffusion-based reconstruction with LOUPE optimized sampling patterns using our objective} (\textit{Diffusion+Our Loss+LOUPE}): Our proposed objective, described in Sec.~\ref{sec:loss}, enables application of LOUPE \cite{multicoil_loupe}, a previously proposed sampling pattern optimization scheme, for diffusion-based reconstruction. The sampling pattern is optimized with the same frozen, pre-trained diffusion model using Adam \cite{kingma2015adam} with a learning rate of $10^{-3}$ for 5 epochs and batch size of 1.

\noindent \textbf{Diffusion-based reconstruction with optimized sampling patterns using the proposed learning strategy and objective} (\textit{Ours Full}): This represents the full proposed method, combining our loss formulation with the proposed sampling pattern construction strategy.

We also detail the end-to-end, MoDL-based\cite{aggarwal2018modl} methods presented for contextualizing performance, rather than direct comparison. All MoDL variants consist of an unrolled architecture with 6 unrolls and 6 inner CG steps per unroll, a shared-parameter U-Net with 7.6M parameters, and are optimized using Adam with a learning rate of $10^{-4}$ for 5 epochs and batch size of 1.

\noindent \textbf{MoDL with variable density Poisson Disc sampling patterns} (\textit{MoDL+Poisson}): MoDL trained with fixed variable-density Poisson Disc sampling.

\noindent \textbf{MoDL with LOUPE sampling patterns} (\textit{MoDL+LOUPE}): Joint optimization of MoDL reconstruction and LOUPE sampling patterns.

\section{Results}

\subsection{Comparing Reconstruction Quality}

Fig. \ref{fig:combined_recons} presents slices reconstructed from k-space data under-sampled by $R=20$ with diffusion-models using a Poisson mask and masks optimized with our loss + LOUPE and our full method for PDFS knee, T2 brain, and PD knee data. Our full method achieves improved quantitative metrics, and the zoomed-in areas indicate that our proposed reconstruction produces slices with more fine structural details than those reconstructed with the other diffusion model-based methods. Similarly Fig. \ref{fig:pd_recon} and Fig. \ref{fig:brain_recon} show reconstructed slices, along with $10\times$ error maps, at $R=4$ and $R=16$ for T2 brain and PD knee data respectively. Again, our method, with its optimized sampling pattern, achieves improved quantitative and qualitative performance in comparison to the competing diffusion-based methods.

\begin{figure*}[!t]
    \centering
    \includegraphics[width=\linewidth]{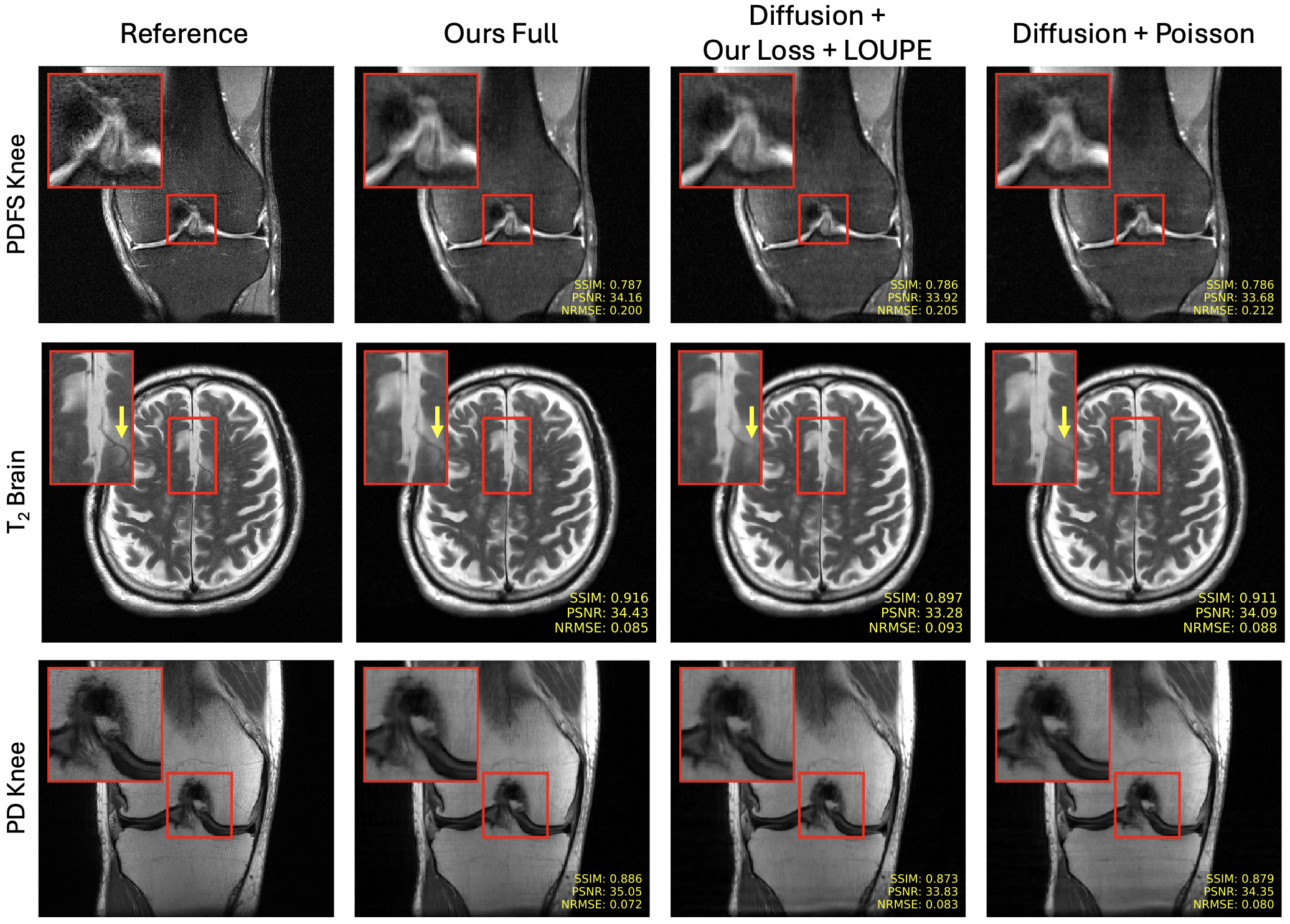}
    \caption{\textbf{Example Reconstructions from different diffusion methods.} We present reconstructions of PDFS knee slice, T2 brain slice, and PD knee slice, all from the test set at $R=20$. Each subfigure includes a zoomed-in region to show fine details (top-left corner) and quantitative metrics (lower-right corner). Reconstructions from our full method, which combines the proposed objective and sampling optimization strategy, achieves better quantitative and qualitative performance in comparison to the baseline diffusion methods.}
    \label{fig:combined_recons}
\end{figure*}

\begin{figure*}[!t]
    \centering
    \includegraphics[width=0.99\linewidth]{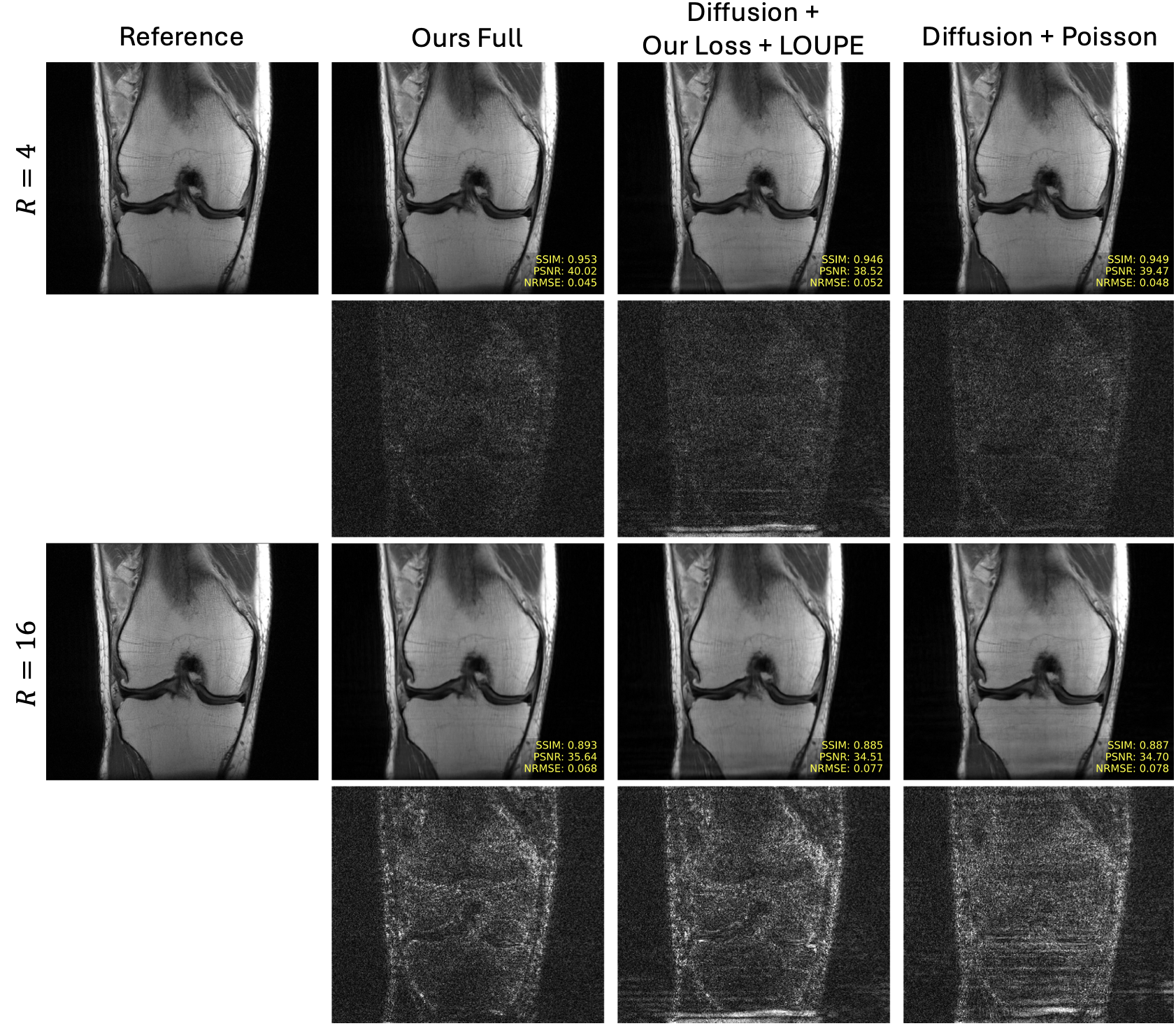}
    \caption{\textbf{Reconstructions from different diffusion methods for a PD knee slice from the test set at R=4 and R=16.} We display the reconstructed slice along with the corresponding residual image (scaled 10$\times$). We present quantitative metrics for each reconstructed slice in the lower-right corner. Our full method produces reconstructions with fewer errors than reconstructions from baseline diffusion methods.}
    \label{fig:pd_recon}
\end{figure*}

\begin{figure*}[!t]
    \centering
    \includegraphics[width=0.9\linewidth]{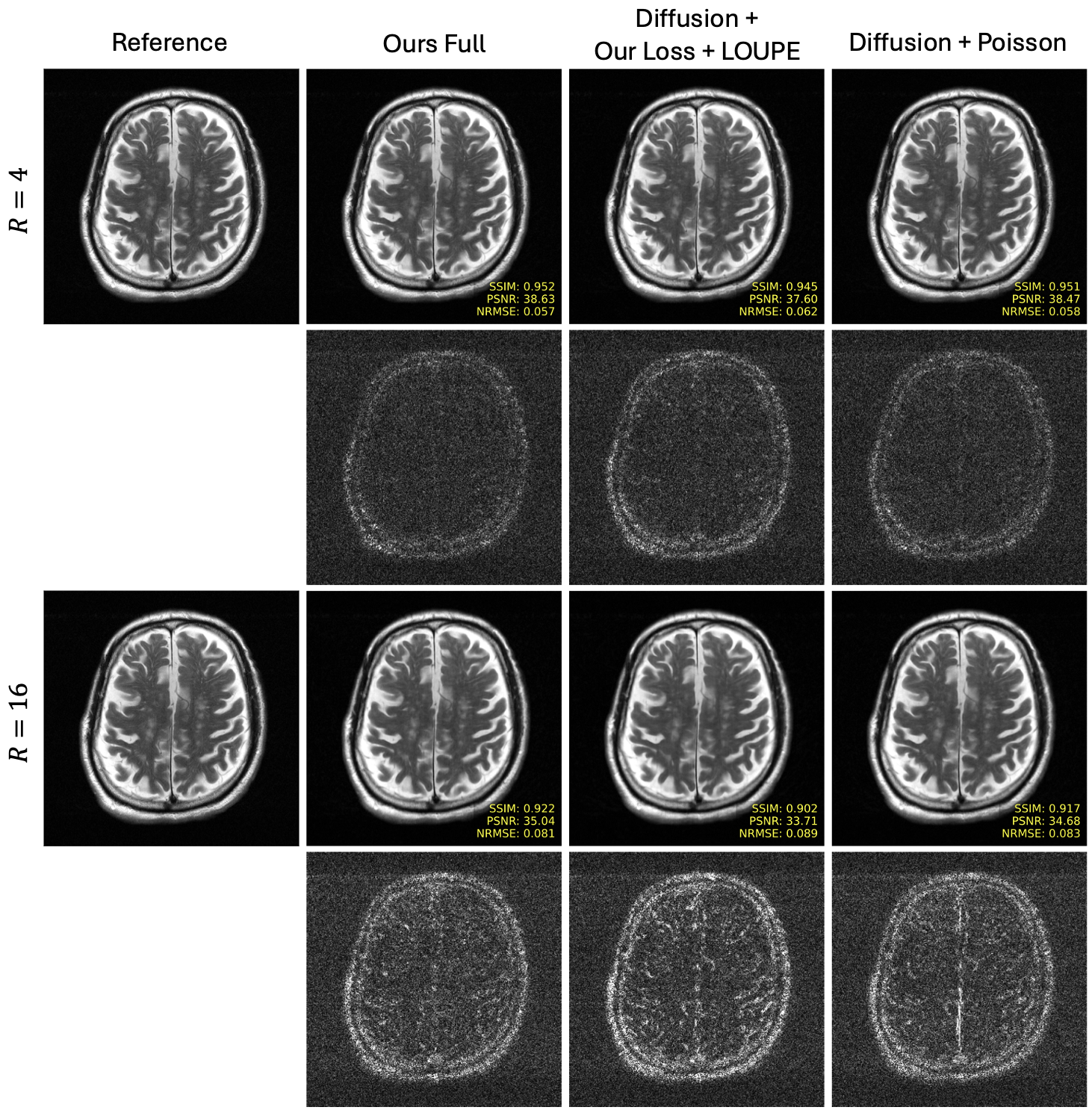}
    \caption{\textbf{Reconstructions from different diffusion methods for a T2 brain slice from the test set at R=4 and R=16.} We display the reconstructed slice along with the corresponding residual image (scaled 10$\times$). We present quantitative metrics for each reconstructed slice in the lower-right corner. Our full method produces reconstructions with fewer errors than reconstructions from baseline diffusion methods.}
    \label{fig:brain_recon}
\end{figure*}

Fig. \ref{fig:metrics} presents quantitative comparisons of peak signal-to-noise ratio (PSNR), structural similarity index measure (SSIM) \cite{ssim} and normalized-root-mean-square-error (NRMSE) of reconstructed slices with acceleration factors of $R=\{4, 8, 12, 16, 20\}$. Our full method outperforms the other two diffusion-based methods across all anatomies, metrics, and accelerations, suggesting that our approach for learning sampling patterns may be better suited for diffusion-based reconstructions in comparison to Poisson Disc patterns and multi-coil LOUPE. Quantative MoDL results are also presented for reference. Since MoDL optimizes an NRMSE objective end-to-end during training, it generally achieves similar or lower NRMSE and PSNR in comparison to the diffusion-based methods. However, the proposed diffusion-based method achieves better SSIM than the end-to-end methods in all settings but PD knee data, where the two methods perform comparably. 

\begin{figure*}[t!] 
    \centering
    \begin{subfigure}[b]{0.23\linewidth}
        \centering
        \includegraphics[width=\linewidth]{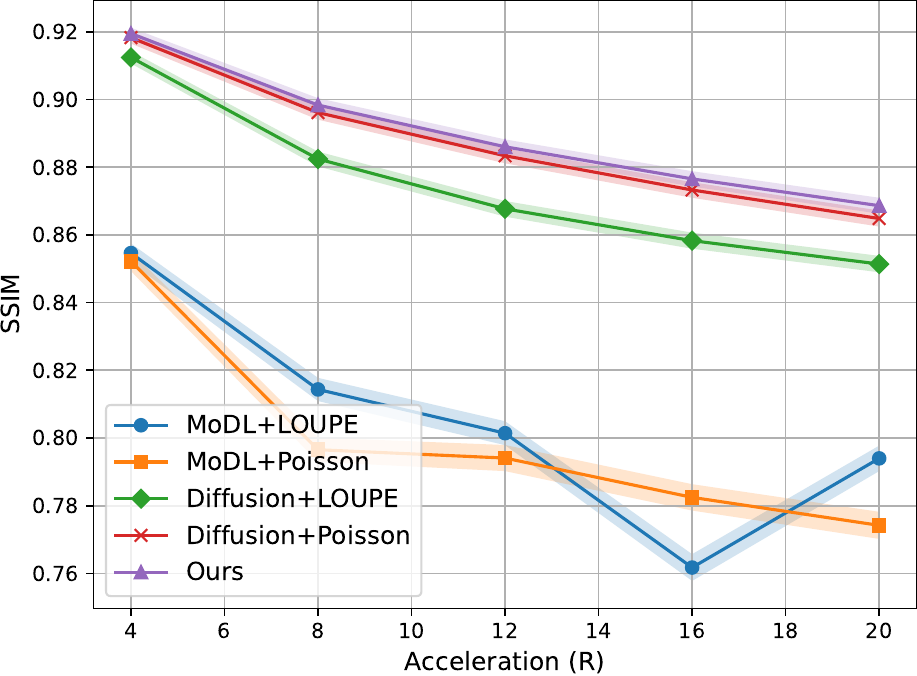}\vspace{0.15in}
        \includegraphics[width=\linewidth]{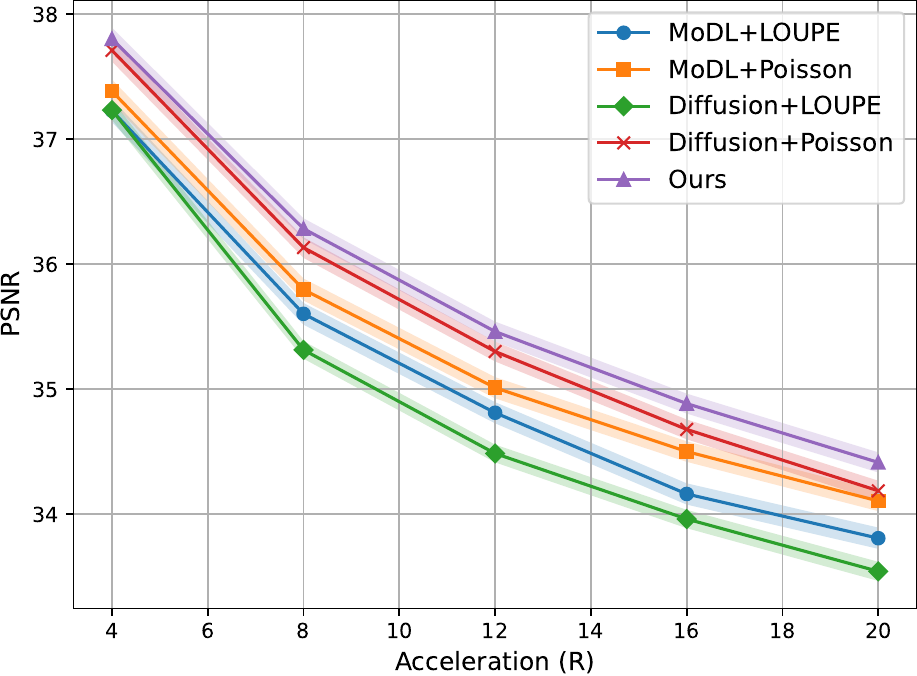}\vspace{0.15in}
        \includegraphics[width=\linewidth]{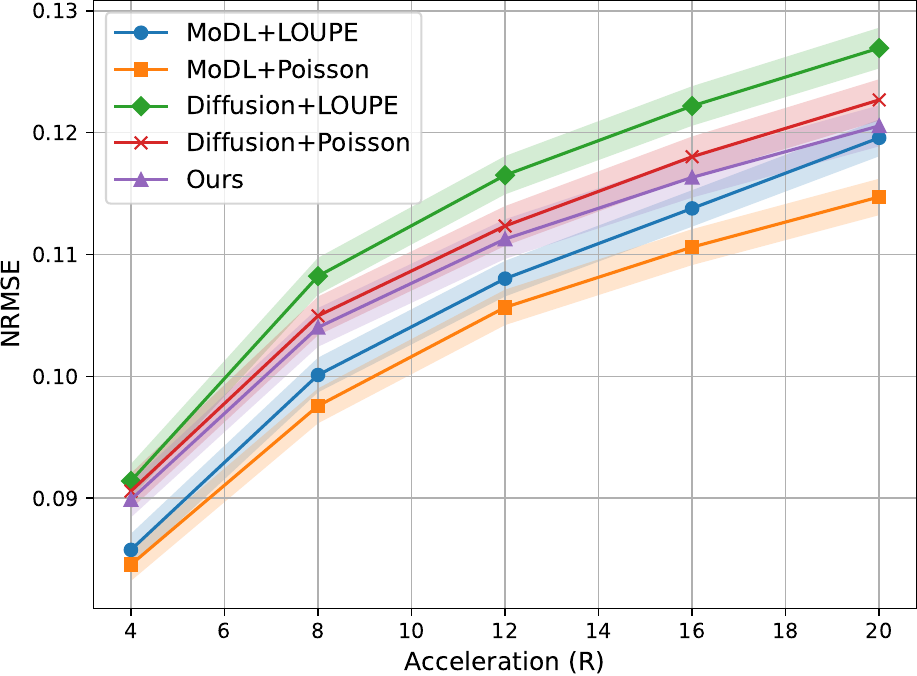}
        \caption{T2 Brain}
        \label{fig:brain_metrics}
    \end{subfigure}
    \hfill
    \begin{subfigure}[b]{0.23\linewidth}
        \centering
        \includegraphics[width=\linewidth]{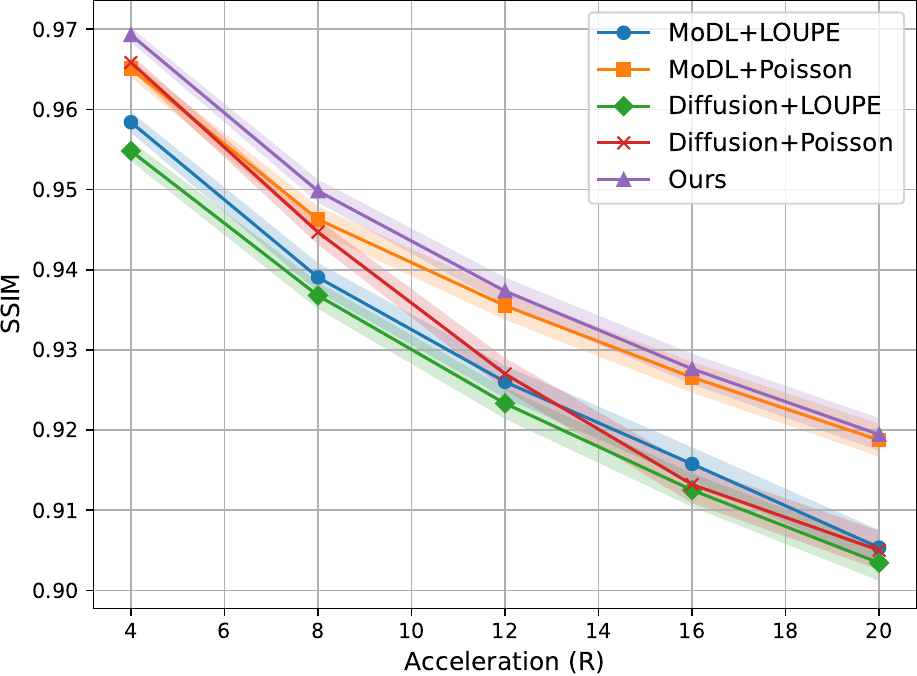}\vspace{0.15in}
        \includegraphics[width=\linewidth]{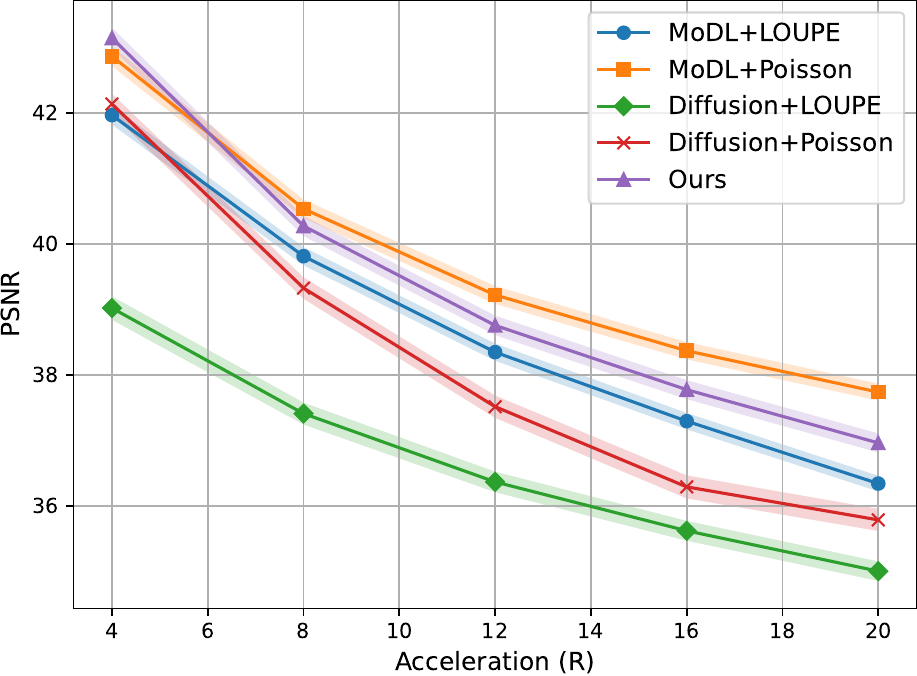}\vspace{0.15in}
        \includegraphics[width=\linewidth]{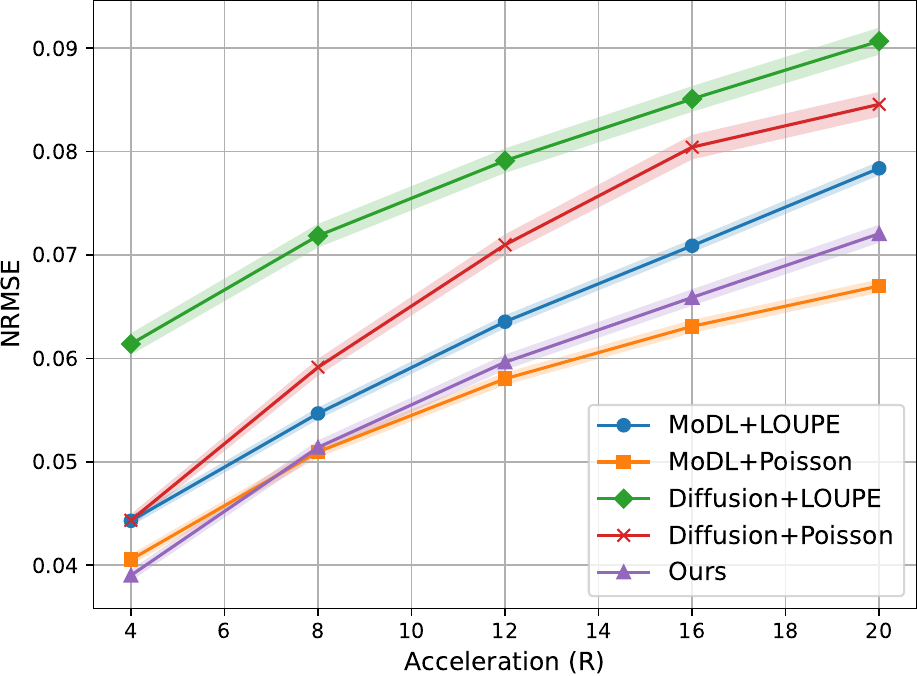}
        \caption{PD Knee}
        \label{fig:pd_metrics}
    \end{subfigure}
    \hfill
    \begin{subfigure}[b]{0.23\linewidth}
        \centering
        \includegraphics[width=\linewidth]{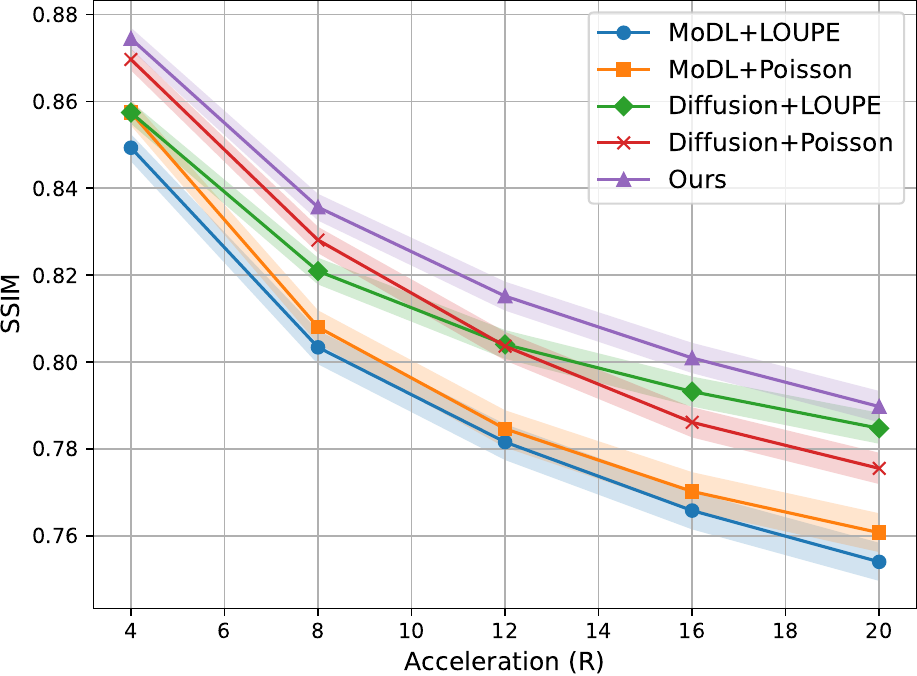}\vspace{0.15in}
        \includegraphics[width=\linewidth]{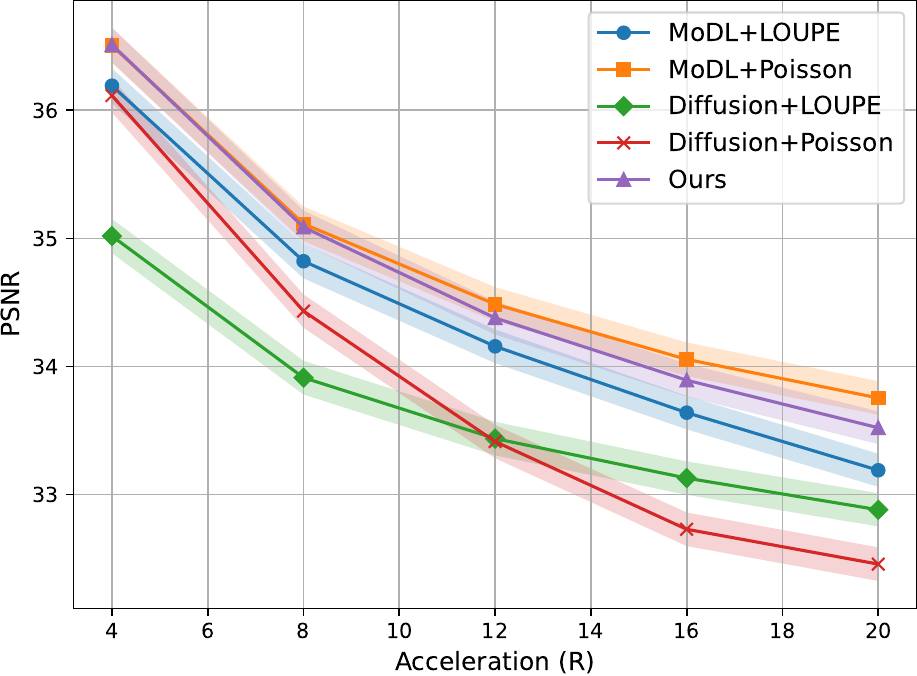}\vspace{0.15in}
        \includegraphics[width=\linewidth]{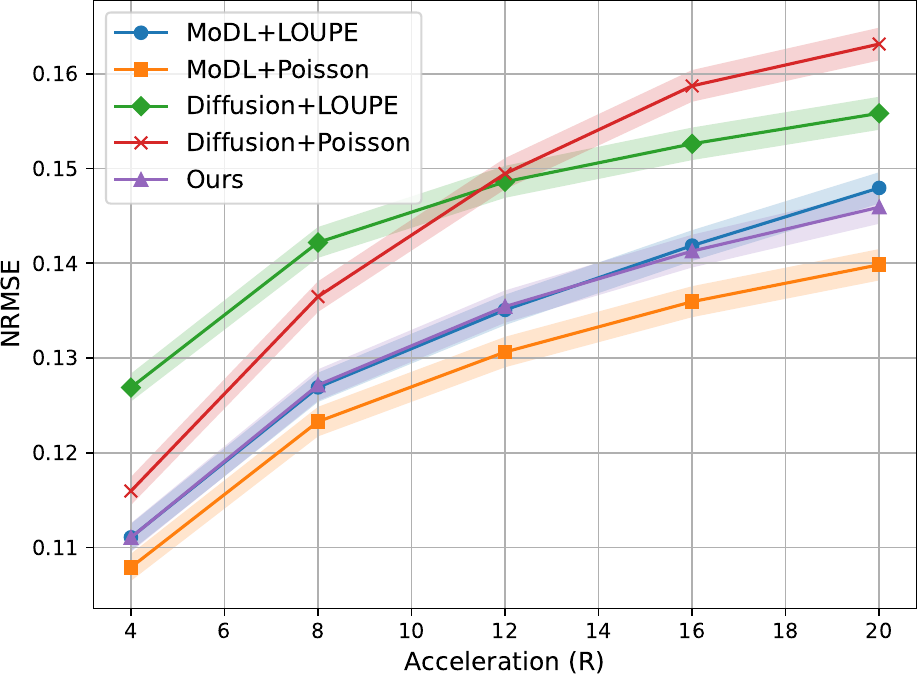}
        \caption{PDFS Knee}
        \label{fig:pdfs_metrics}
    \end{subfigure}
    \hfill
    \begin{subfigure}[b]{0.23\linewidth}
        \centering
        \includegraphics[width=\linewidth]{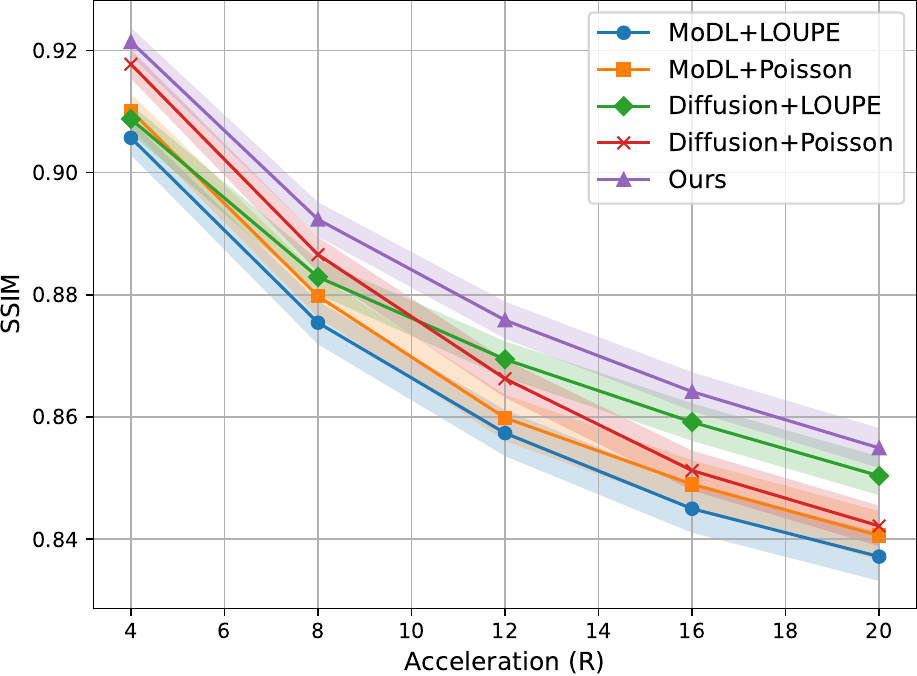}\vspace{0.15in}
        \includegraphics[width=\linewidth]{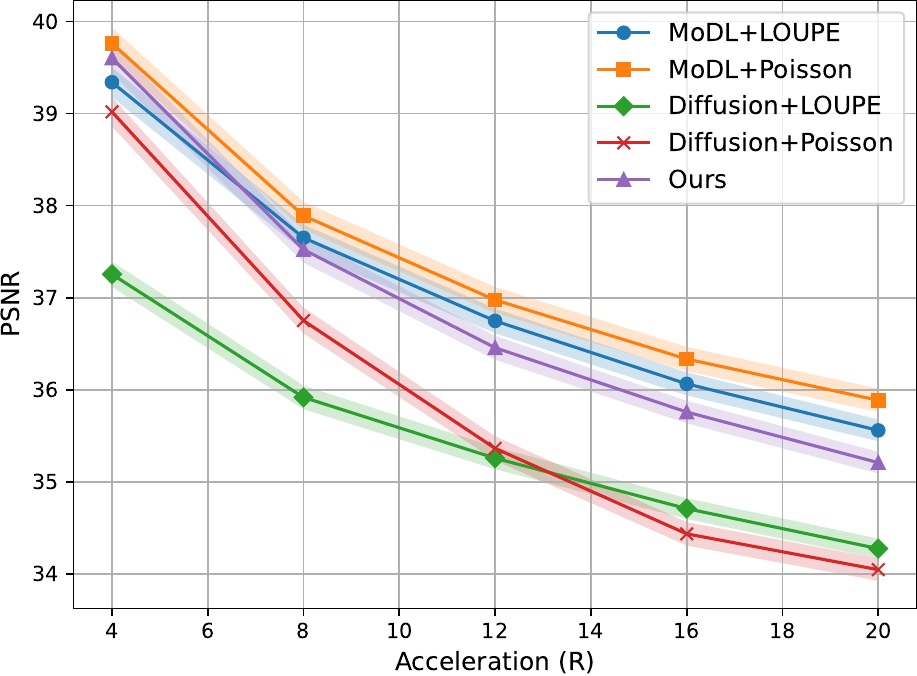}\vspace{0.15in}
        \includegraphics[width=\linewidth]{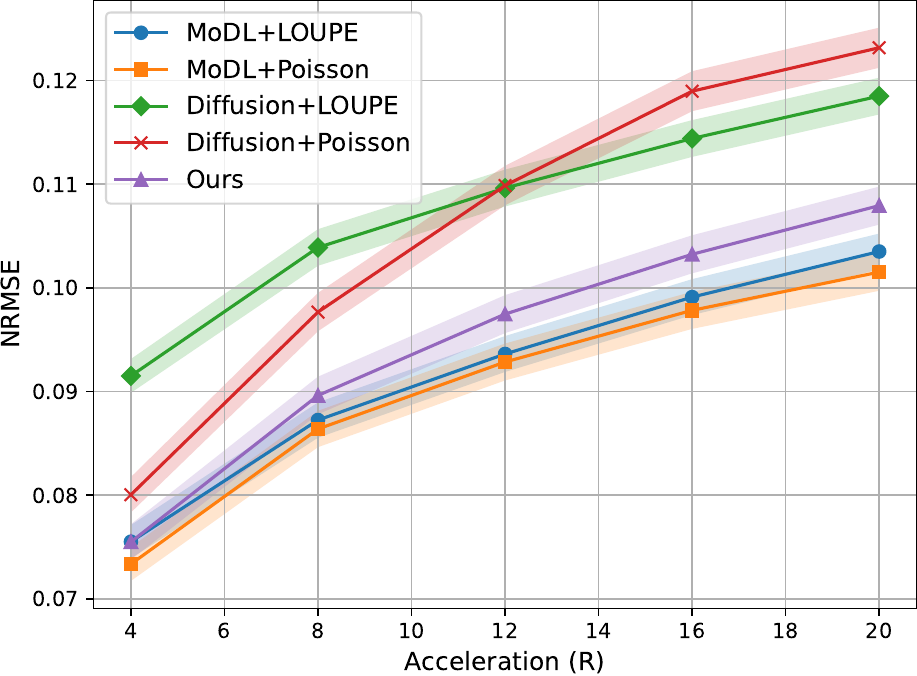}
        \caption{All Knee}
        \label{fig:combined_metrics}
    \end{subfigure}
    \caption{\textbf{Test metrics for different methods across various anatomies and acceleration factors.} We compare the mean SSIM, PSNR, and NRMSE for reconstructions of slices from the test set using our method vs. baselines. We present results for acceleration factors $R$ in $\{4, 8, 12, 16, 20\}$ for (a) T2 brain, (b) PD knee, (c) PDFS knee, and (d) combined PD and PDFS knee data. The shaded areas indicate $\pm 1$ standard error. Across metrics, accelerations, and anatomy, our diffusion-based reconstruction with optimized sampling patterns outperforms the baseline variable-density Poisson Disc sampling and multi-coil LOUPE diffusion-based reconstructions. As the goal of this work is not to evaluate the performance of diffusion versus end-to-end reconstruction methods, the MoDL metrics are just included as reference points for contextualization.}
    \label{fig:metrics}
\end{figure*}

\subsection{Comparing Sampling Patterns}

Next, we compare our learned patterns for diffusion reconstruction to baseline patterns visually for $R=4, 8,$ and $16$ for T2 brain data in Fig.~\ref{fig:brain_patterns} and PDFS knee data in Fig.~\ref{fig:pdfs_patterns}. Our patterns adapt to the data distribution, with distinct spreads of samples between the brain and knee patterns.
Due to our greedy selection process, the points in our learned patterns maintain reasonable separation from each other and avoid clustering, while still densely sampling informative central regions of k-space. In contrast, patterns learned with LOUPE show many regions of clustered k-space points that may capture redundant information. Further, LOUPE patterns trained on PDFS knee data exhibit artifacts; portions of k-space along the central vertical axis are very sparsely sampled. We carefully verify that these artifacts are not caused by implementation errors or discrepancies in experiment code between our method and LOUPE. Finally, our patterns adapt gracefully to different acceleration factors, distributing sampling points more evenly across k-space with increasing $R$ while still favoring informative regions.

\begin{figure}[!t]
    \centering
    \includegraphics[width=0.99\linewidth]{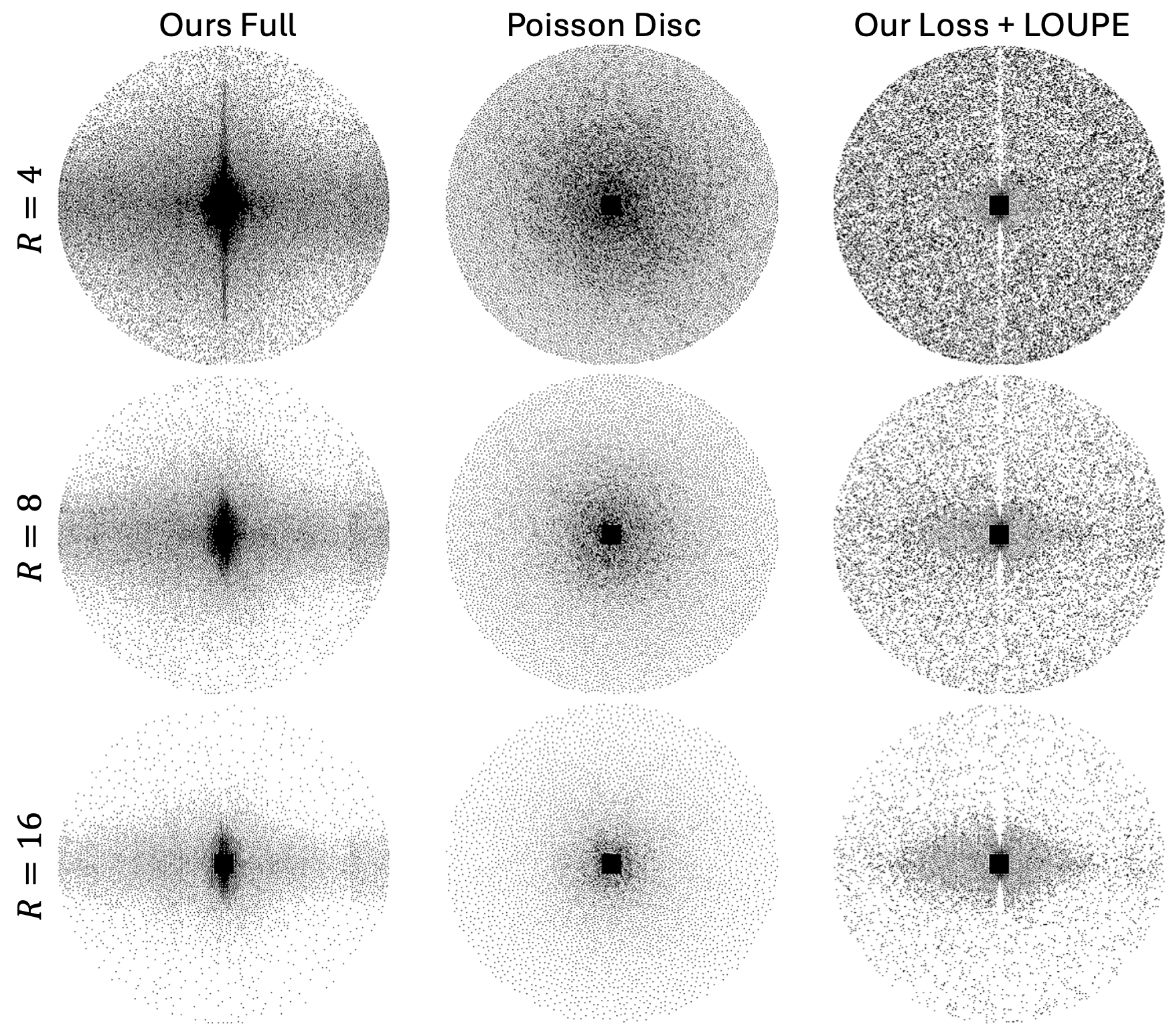}
    \caption{\textbf{Sampling patterns from different methods for PDFS knee data.} Sampling patterns used for reconstructing slices at acceleration factors $R$ in $\{4, 8, 16\}$ with the diffusion-based methods. Poisson Disc sampling tends to distribute samples more uniformly at high acceleration rates, while LOUPE concentrates many samples near the center of k-space. The proposed method balances these effects by maintaining spacing between samples while preferentially allocating measurements to informative central regions.}
    \label{fig:pdfs_patterns}
\end{figure}

\subsection{Effect of Contrast}
\label{sec:contrast}

Here, we analyze the effects of learning sampling patterns with different contrasts for the same anatomy. We train our method on knee data with PD, PDFS, and combined PD and PDFS contrasts for $R=4, 8,$ and $16$ and visualize the learned sampling patterns in Fig.~\ref{fig:knee_comparison_patterns}. At low acceleration, $R=4$, the pattern learned on PDFS data samples low-frequency components more densely than the pattern for PD data, particularly along the vertical and horizontal axes about the origin. PDFS data have a lower signal-to-noise ratio than PD data, as indicated by the worse overall PSNR and NRMSE values in the plots in Fig.~\ref{fig:metrics}. This suggests that our proposed method is correctly learning to concentrate on the most informative regions of k-space for PDFS data while sampling the noisier high-frequency components more sparsely. At $R=8$ and $16$, the patterns learned on each contrast begin to converge in appearance, with the PDFS pattern still samples low-frequency components more densely than the PD pattern. The patterns learned on the combined contrast data appear to be a visual average between the patterns learned on each contrast separately, indicating that our proposed technique properly adapts to the combined data. 

\begin{figure}[!t]
    \centering
    \includegraphics[width=0.99\linewidth]{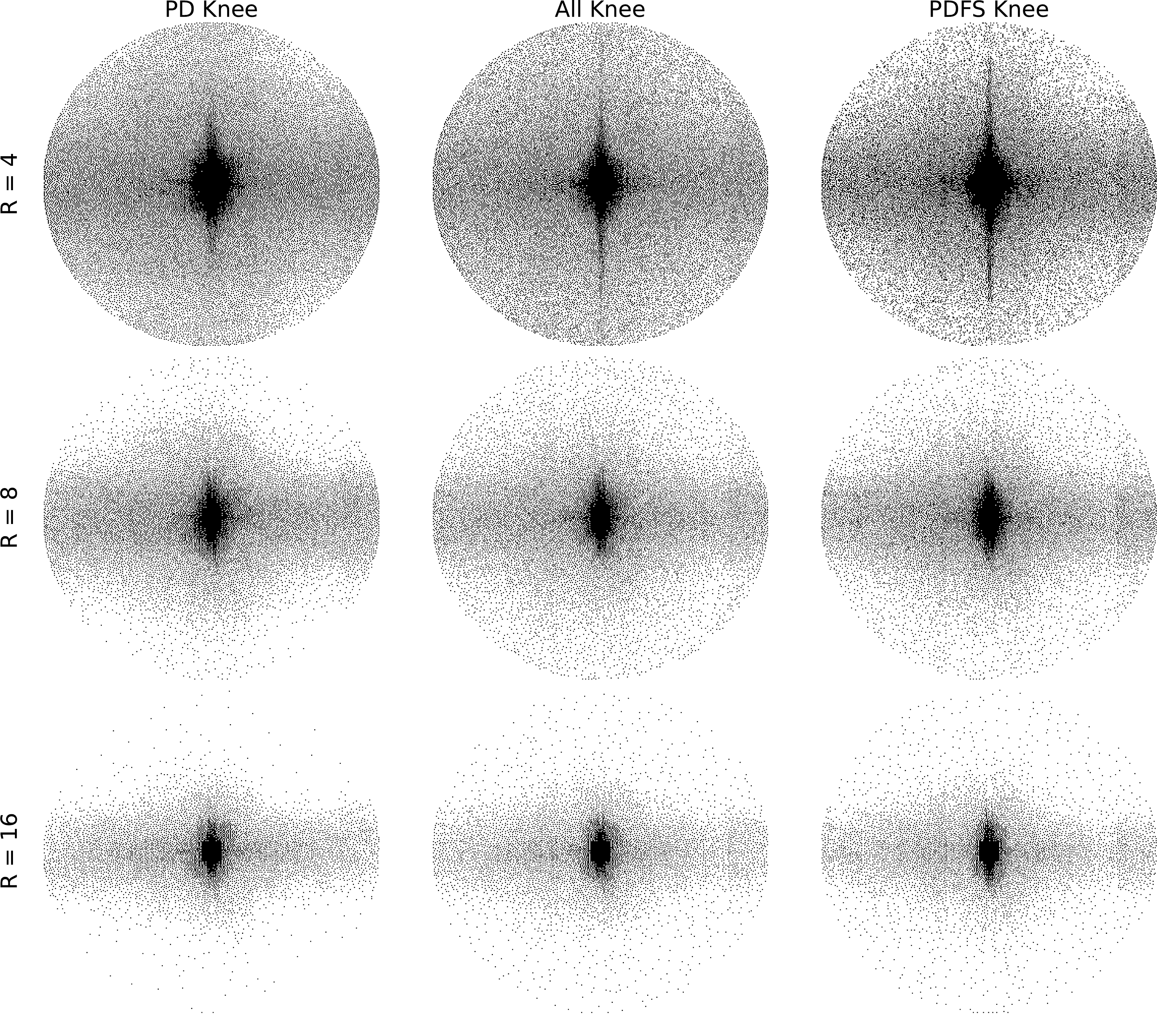}
    \caption{\textbf{Effects of different contrast on patterns learned with our method.} We train sampling patterns using our method on PD and PDFS knee slices separately, as well as on both contrasts combined. We present the learned patterns for acceleration factors $R$ in $\{4, 8, 16\}$. Our method adapts to changes in data distribution caused by shifts in contrast, even for the same anatomy.}
    \label{fig:knee_comparison_patterns}
\end{figure}

\subsection{Effect of Top-K Parameter}

Finally, in Fig.~\ref{fig:brain_patterns_varyingk} we plot sampling patterns learned on T2 brain data using our method with varying $K$ at $R=4, 8,$ and $16$. As we expect, increasing $K$ leads to more separation between neighboring k-space samples and a more uniform distribution of overall sampling points. By interpolating between pure greedy sample selection at $K=1$ and distance-based selection at higher values of $K$, we can finely control the tradeoff between structure and randomness in our learned patterns. We present further results on the effects of varying the top-K parameter in the Appendix.  

\begin{figure*}[!t]
    \centering
    \includegraphics[width=0.7\linewidth]{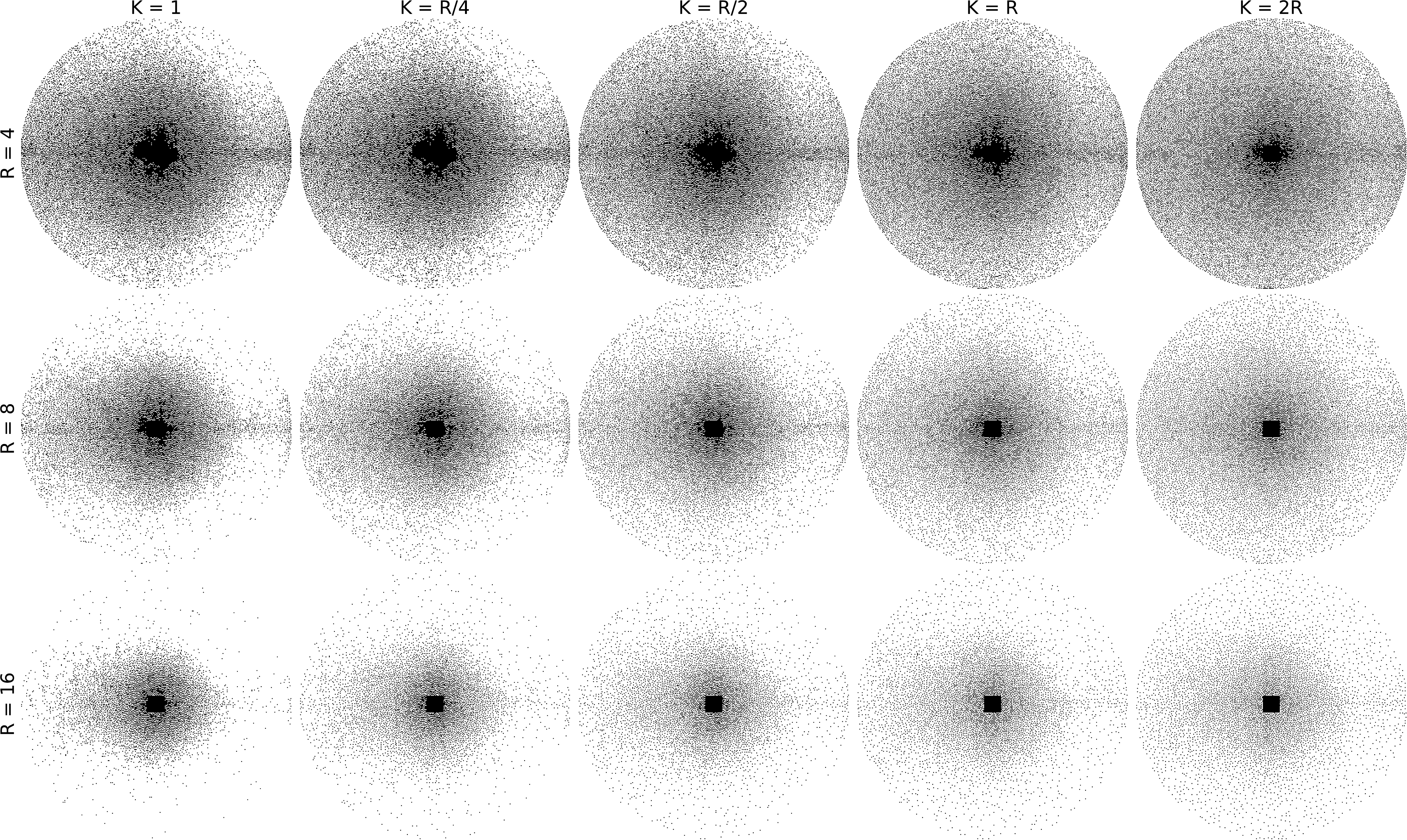}
    \caption{\textbf{Effects of varying the top-K parameter on patterns learned with our method.} We present sampling patterns learned using our method on T2 brain slices for acceleration factors $R$ in $\{4, 8, 16\}$ and top-K parameter $K$ in $\{1, \frac{R}{4}, \frac{R}{2}, R, 2R\}$. Increasing $K$ leads to patterns with greater distance between neighboring k-space samples.}
    \label{fig:brain_patterns_varyingk}
\end{figure*}

\section{Discussion}

This work proposes a framework for optimizing k-space sampling patterns for diffusion-based MRI reconstruction. The first contribution is a training objective that enables backpropagation of gradients with respect to the sampling pattern in the setting of diffusion-model reconstruction, where the computational burden of posterior sampling typically prohibits differentiation. Then, we introduce a greedy sampling pattern optimization algorithm that exploits our objective to iteratively select k-space locations. An additional practical advantage of the proposed approach is that sampling patterns are optimized using a frozen, pre-trained diffusion prior, allowing the method to be flexibly applied across acquisition settings without retraining the reconstruction model or overfitting the prior to a specific sampling mask, as commonly occurs in joint end-to-end optimization.

Across experiments, the proposed sampling patterns consistently outperform both variable-density Poisson Disc sampling and multi-coil LOUPE when used with diffusion-based reconstruction. Poisson Disc patterns were originally designed to promote sparsity-based reconstruction, while LOUPE was developed for unrolled, end-to-end networks with joint training of reconstruction and sampling. In contrast, our method maintains structure by selecting candidate k-space samples based on their gradient values while also choosing the candidate that is furthest away from the existing samples. As shown in Figs. \ref{fig:brain_patterns} and \ref{fig:pdfs_patterns}, Poisson Disc sampling tends to distribute samples more uniformly at high acceleration rates, which can lead to insufficient coverage of informative low-frequency regions, while LOUPE concentrates many samples near the center of k-space, potentially resulting in redundant measurements. The proposed method balances these effects by maintaining spacing between samples while preferentially allocating measurements to informative central regions. Furthermore, the learned patterns adapt to contrast and anatomy (Fig. \ref{fig:knee_comparison_patterns}).

In several cases, NRMSE and PSNR achieved by the proposed diffusion-based approach are similar to or slightly worse than those obtained by end-to-end MoDL reconstructions. However, the goal of this work is not to compare diffusion-based reconstruction against end-to-end methods, but rather to establish an effective sampling optimization strategy for diffusion models and demonstrate improvements over prior diffusion-based baselines. Moreover, MoDL is explicitly trained to minimize NRMSE, and we observe that MoDL reconstructions are often over-smoothed and exhibit loss of fine structural detail. This is reflected by our method achieving higher SSIM than either MoDL-based baseline for all but PD knee data, where MoDL+Poisson performs similarly. A deeper investigation of the perception–distortion tradeoff \cite{blau2018perception} between diffusion-based and end-to-end reconstructions is an important direction for future work but is beyond the scope of this study.

Interestingly, we observe that at higher acceleration rates, both our method and LOUPE allocate a larger fraction of samples to the center of k-space. Recent work on MRI denoising has shown that, under a fixed scan-time budget, acquiring lower-resolution, higher-SNR measurements and denoising them can be more effective than attempting to de-noise noisy high-resolution data \cite{wang2025welldesigned}. We speculate that a similar phenomenon may be occurring in the reconstruction setting: the optimizer appears to favor concentrating measurements in low-frequency regions, effectively biasing toward lower-resolution but more reliable observations, and then relying on the reconstruction model to recover fine-scale structure. Thoroughly characterizing this effect for learned reconstruction and sampling pattern optimization is an interesting direction for future research.

A limitation of the proposed approach is the requirement of fully sampled datasets, both for training the diffusion model and optimizing sampling patterns. This constraint is shared by many learning-based reconstruction methods, but could potentially be relaxed using approaches that operate on partially observed data, like Ambient Diffusion \cite{aali2024ambient}. Additionally, while the proposed objective could, in principle, be applied to end-to-end reconstructions for joint learning of sampling and reconstruction, doing so may be ineffective in practice due to the instability of early training gradients when the weights of reconstruction model have not converged and the irreversible nature of the greedy sample selection process.

Future extensions of our work could combine techniques for learning sampling pattern distributions with our greedy selection scheme. For example, we could learn a Bernoulli probability for each k-space location in a primary training phase, then apply our method in a second phase, using the learned probabilities to inform our choice of candidate k-space points instead of picking them uniformly at random. Another direction could be to parameterize a higher-dimensional sampling operator that explicitly models, for example, interactions between different coils or k-space locations.  

\section{Conclusion}

We proposed a method to learn sampling patterns for MRI reconstruction using diffusion-based generative models and posterior sampling. To alleviate computational and memory expense during training, we derived an extension to Tweedie's formula that provides a closed-form expression for the mean of the posterior distribution, leading us to a simplified training objective. We then presented a technique for learning sampling patterns using a greedy two-step process that allows the user to tune the tradeoff between structured and non-structured sampling. Empirical evaluations demonstrated that our method outperforms baselines diffusion-based methods for reconstructing MRI data across various anatomies, contrasts, and acceleration factors. 

\appendix    
\input{appendix}

\section*{Disclosures}

The authors declare that there are no financial interests, commercial affiliations, or other potential conflicts of interest that could have influenced the objectivity of this research or the writing of this paper

\section*{Data and Code Availability Statement}

All data for this work was taken from the publicly available fast MRI dataset. Code to optimize the sampling pattern with the full proposed method can be found at: \url{https://github.com/utcsilab/MRI_Diffusion_Patterns}.

\section*{Acknowledgments}

This research has been supported by NSF Grants CCF-2239687 (CAREER), AF 1901292, CNS 2148141, Tripods CCF 1934932, IFML CCF 2019844, and research gifts by Western Digital, Amazon, WNCG IAP, UT Austin Machine Learning Lab (MLL), Cisco, and the Stanly P. Finch Centennial Professorship in Engineering.

\bibliography{references}   
\bibliographystyle{spiejour}   




\listoffigures
\listoftables

\end{spacing}
\end{document}

%% file: appendix.tex
\clearpage

\section{Proof of \Cref{prop:tweedie}}
\label{sec:appendix_tweedie_proof}

\tweedie*

\begin{proof}
We begin by representing the distribution $p_t(\x_t | \y)$ as marginalizing out $\x_0$ conditioned on $\y$:
    \begin{align*}
        p_t(\x_t | \y) = \int_{\x_0} p_t(\x_t | \x_0, \y) p_0(\x_0 | \y) d\x_0.
    \end{align*}
Next, we take the gradient w.r.t. $\x_t$ on both sides:
    \begin{align*}
       \nabla_{\x_t} p_t(\x_t | \y) &= \nabla_{\x_t} \int_{\x_0} p_t(\x_t | \x_0, \y) p_0(\x_0 | \y) d\x_0 \\
                                    &= \int_{\x_0} p_0(\x_0 | \y) \nabla_{\x_t} p_t(\x_t | \x_0, \y) d\x_0 \\
                                    &= \int_{\x_0} \big{[}p_0(\x_0 | \y) p_t(\x_t | \x_0, \y) \\ &\qquad \qquad  \nabla_{\x_t} \log p_t(\x_t | \x_0, \y) \big{]} d\x_0.
    \end{align*}
On the last line, we use the identity $\nabla_x \log f(x) = \nabla_x f(x) / f(x)$. We note that since $p_t(\x_t | \x_0, \y) = p_t(\x_t | \x_0)$, and $p_t(\x_t | \x_0)$ is Gaussian, $p_t(\x_t | \x_0, \y)$ has non-zero value everywhere and we avoid singularities from the denominator. 

Continuing, we use the conditional independence of $\x_t$ and $\y$ given $\x_0$ to replace $\nabla_{\x_t} \log p_t(\x_t | \x_0, \y)$ on the right-hand side with $\nabla_{\x_t} \log p_t(\x_t | \x_0)$ and obtain:
    \begin{align*}
        \nabla_{\x_t} p_t(\x_t | \y) &= \int_{\x_0} p_0(\x_0 | \y) p_t(\x_t | \x_0, \y) \nabla_{\x_t} \log p_t(\x_t | \x_0) d\x_0 \\
                                     &= \int_{\x_0} p_0(\x_0 | \y) p_t(\x_t | \x_0, \y) \bigg(\frac{\x_0 - \x_t}{\sigma^2_t}\bigg) d\x_0.
    \end{align*}
Here, we use the fact that $p_t(\x_t | \x_0) = \cN(\x_t; \x_0, \sigma^2_t \I_n)$ is a Gaussian and replace the score function $\nabla_{\x_t} \log p_t(\x_t | \x_0)$ by its exact value, $(\x_0 - \x_t)/\sigma^2_t$. 

Expanding the right-hand-side, we get:
    \begin{align*}
        \nabla_{\x_t} p_t(\x_t | \y) &= \frac{1}{\sigma^2_t} \bigg[\int_{\x_0} p_0(\x_0 | \y) p_t(\x_t | \x_0, \y) \x_0 d\x_0 - \\
        &\qquad \qquad \int_{\x_0} p_0(\x_0 | \y) p_t(\x_t | \x_0, \y) \x_t d\x_0 \bigg] \\
                                     &= \frac{1}{\sigma^2_t} \bigg[\int_{\x_0} p_0(\x_0 | \y) p_t(\x_t | \x_0, \y) \x_0 d\x_0 - \\ & \qquad \qquad\x_t \int_{\x_0} p_0(\x_0 | \y) p_t(\x_t | \x_0, \y) d\x_0 \bigg] \\
                                     &= \frac{1}{\sigma^2_t} \bigg[\int_{\x_0} p_0(\x_0 | \y) p_t(\x_t | \x_0, \y) \x_0 d\x_0 - \\&\qquad \qquad \x_t p_t(\x_t | \y) \bigg]. 
    \end{align*}
In the previous line, we marginalize out $\x_0$ conditioned on $\y$ as in the first line of the proof to recover $p_t(\x_t | \y)$. 

Next, we observe that Bayes' rule tells us $p_0(\x_0 | \y) p_t(\x_t | \x_0, \y) = p_t(\x_t | \y) p_0(\x_0 | \x_t, \y)$ and replace the former quantity by the latter on the right-hand side:
    \begin{align*}
        \nabla_{\x_t} p_t(\x_t | \y) &= \frac{1}{\sigma^2_t} \bigg[\int_{\x_0} p_t(\x_t | \y) p_0(\x_0 | \x_t, \y) \x_0 d\x_0 - \\&\qquad \qquad \x_t p_t(\x_t | \y) \bigg] \\
                                     &= \frac{1}{\sigma^2_t} \bigg[ p_t(\x_t | \y) \E[\x_0 | \x_t, \y] - \x_t p_t(\x_t | \y) \bigg] \\
                                     &= \frac{p_t(\x_t | \y)}{\sigma^2_t} \bigg[ \E[\x_0 | \x_t, \y] - \x_t \bigg] \\
        \frac{\nabla_{\x_t} p_t(\x_t | \y)}{p_t(\x_t | \y)} &= \frac{1}{\sigma^2_t} \bigg[ \E[\x_0 | \x_t, \y] - \x_t \bigg].
    \end{align*}
We note that from our assumption, $p_t(\x_t | \y)$ is fully supported everywhere, so we avoid singularities when dividing by this quantity. 

Finally, we again invoke the identity $\nabla_x \log f(x) = \nabla_x f(x) / f(x)$ to rewrite the left-hand side and rearrange to obtain the desired result:
    \begin{align*}
        \nabla_{\x_t} \log p_t(\x_t | \y) &= \frac{1}{\sigma^2_t} \bigg[ \E[\x_0 | \x_t, \y] - \x_t \bigg] \\
        \E[\x_0 | \x_t, \y] &= \x_t + \sigma^2_t \nabla_{\x_t} \log p_t(\x_t | \y).
    \end{align*}
\end{proof}

\section{Further Experiments for Effects of Varying the Top-K Parameter}
\label{sec:appendix_topk}

\begin{figure*}[t!] 
    \centering
    \begin{subfigure}[b]{0.23\linewidth}
        \centering
        \includegraphics[width=\linewidth]{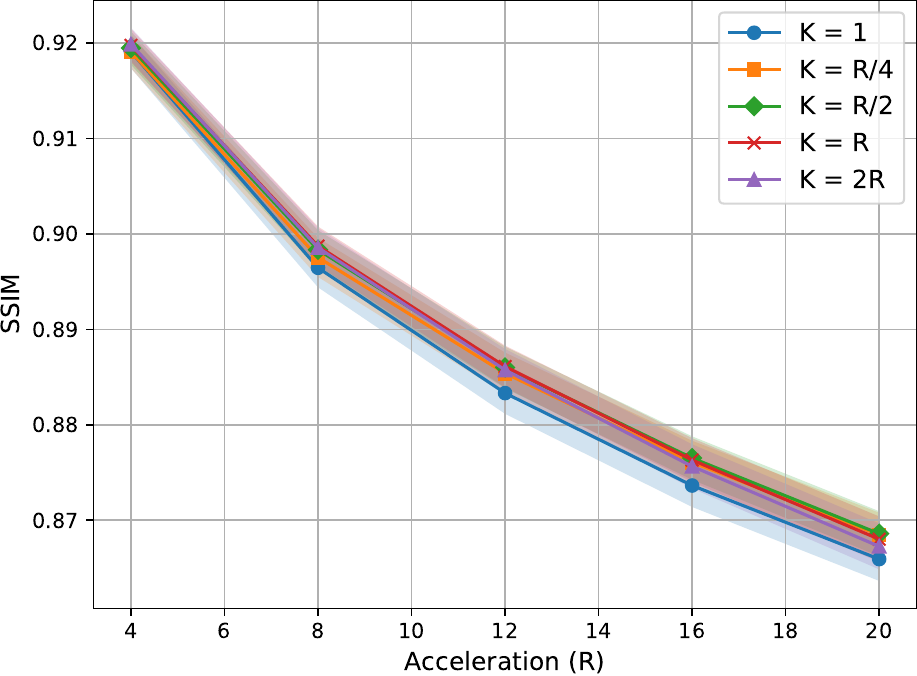}\vspace{0.15in}
        \includegraphics[width=\linewidth]{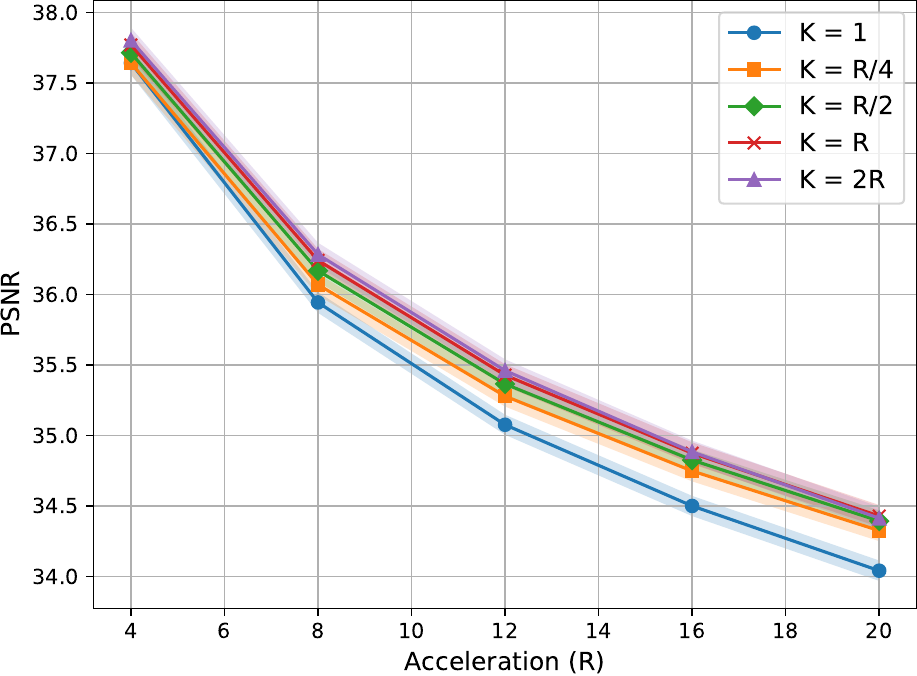}\vspace{0.15in}
        \includegraphics[width=\linewidth]{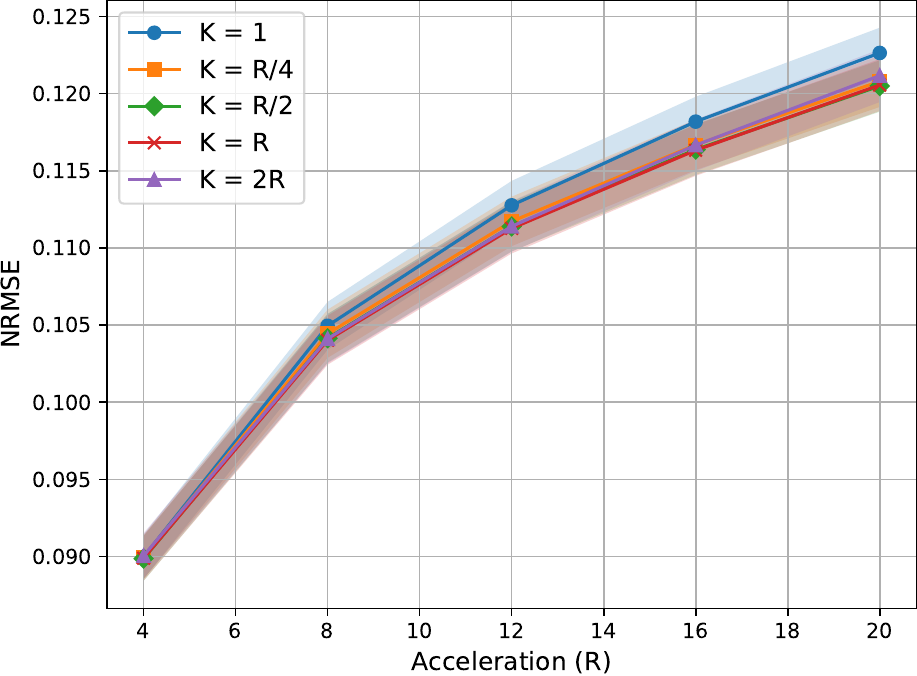}
        \caption{T2 Brain}
        \label{fig:brain_metrics_varyingk}
    \end{subfigure}
    \hfill
    \begin{subfigure}[b]{0.23\linewidth}
        \centering
        \includegraphics[width=\linewidth]{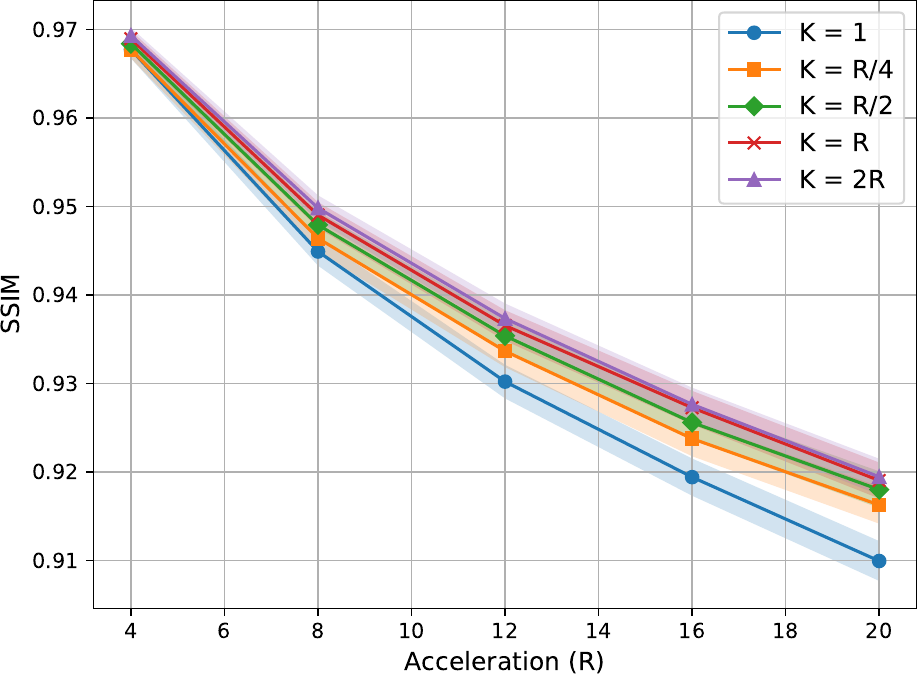}\vspace{0.15in}
        \includegraphics[width=\linewidth]{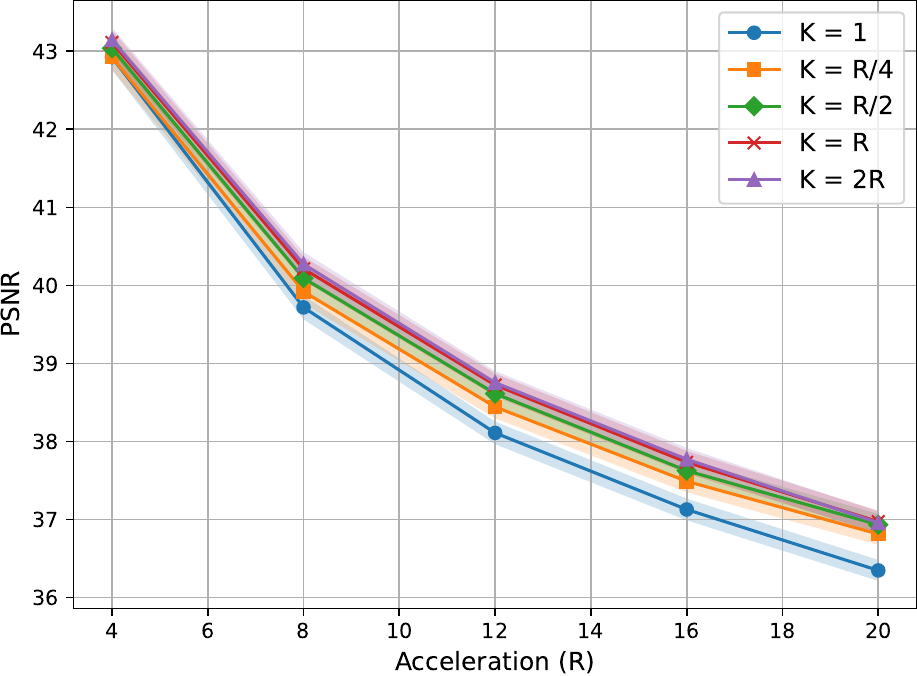}\vspace{0.15in}
        \includegraphics[width=\linewidth]{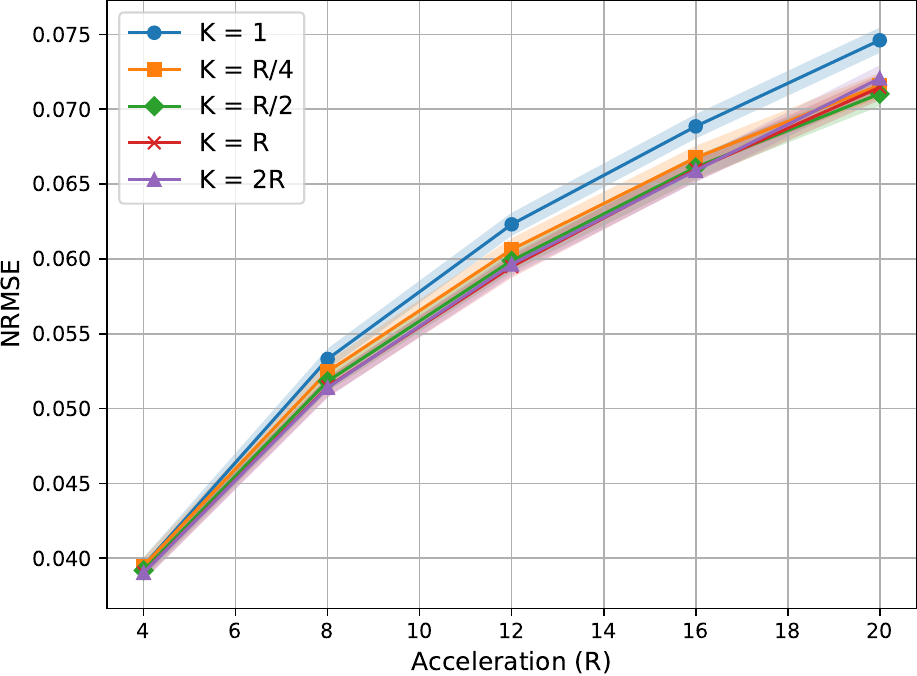}
        \caption{PD Knee}
        \label{fig:pd_metrics_varyingk}
    \end{subfigure}
    \hfill
    \begin{subfigure}[b]{0.23\linewidth}
        \centering
        \includegraphics[width=\linewidth]{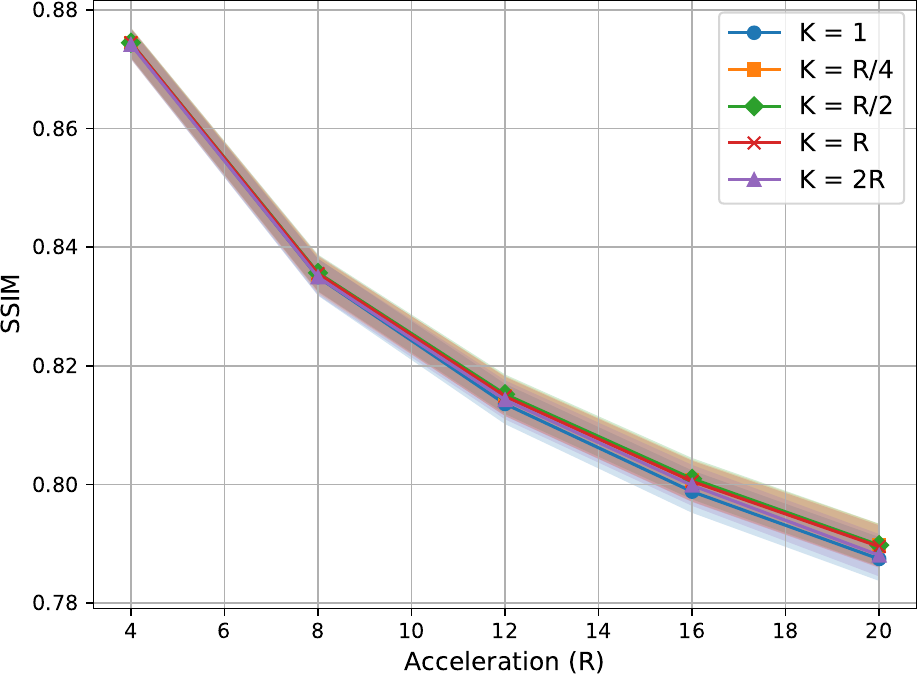}\vspace{0.15in}
        \includegraphics[width=\linewidth]{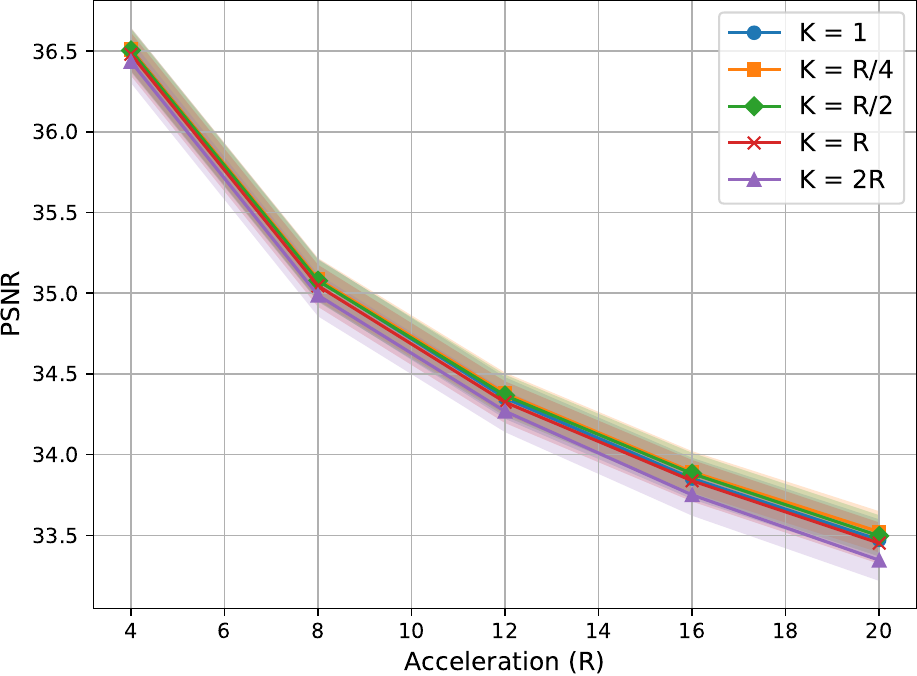}\vspace{0.15in}
        \includegraphics[width=\linewidth]{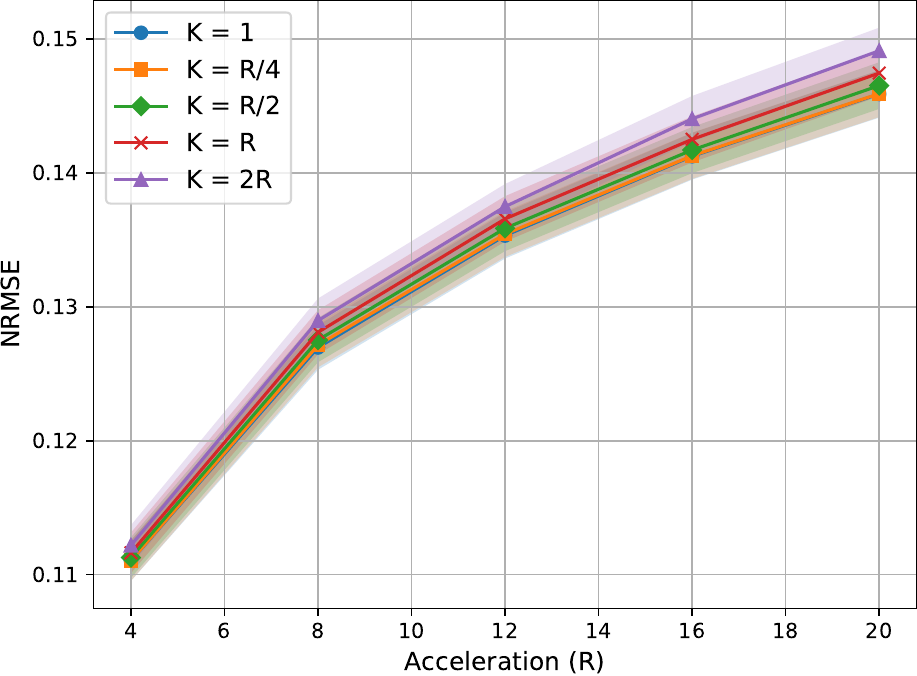}
        \caption{PDFS Knee}
        \label{fig:pdfs_metrics_varyingk}
    \end{subfigure}
    \hfill
    \begin{subfigure}[b]{0.23\linewidth}
        \centering
        \includegraphics[width=\linewidth]{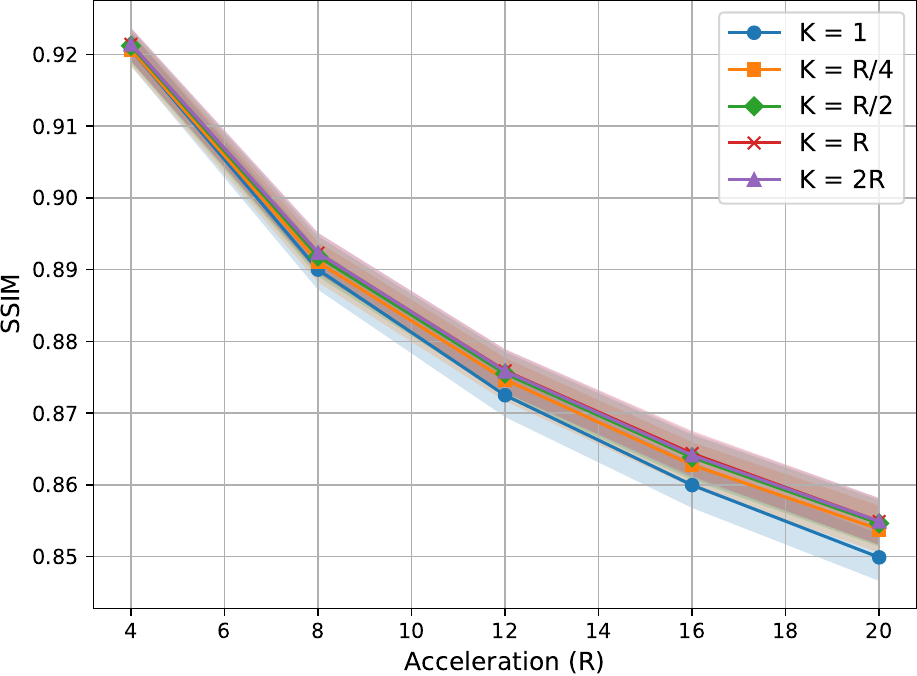}\vspace{0.15in}
        \includegraphics[width=\linewidth]{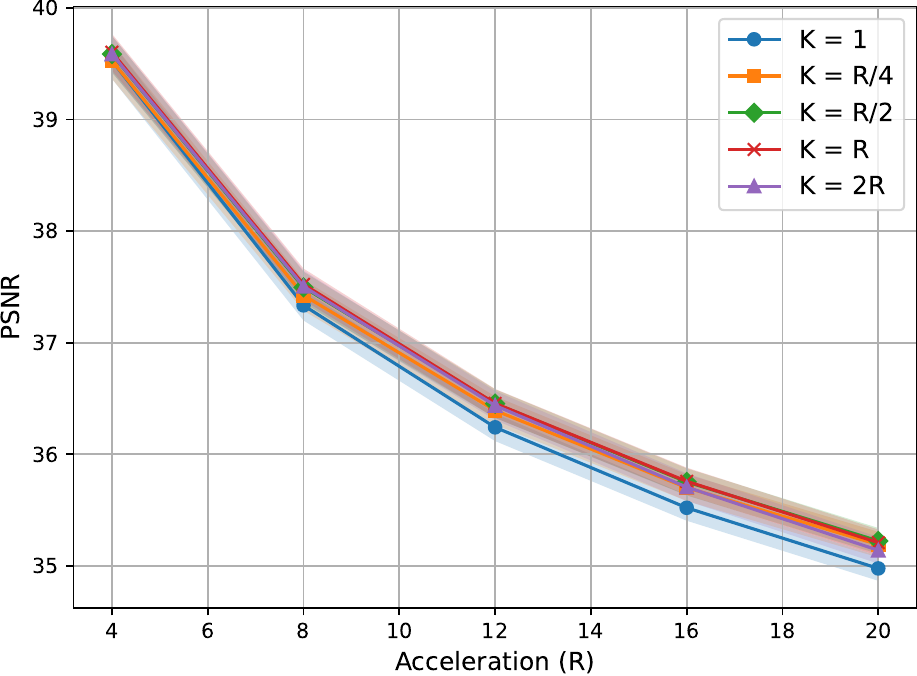}\vspace{0.15in}
        \includegraphics[width=\linewidth]{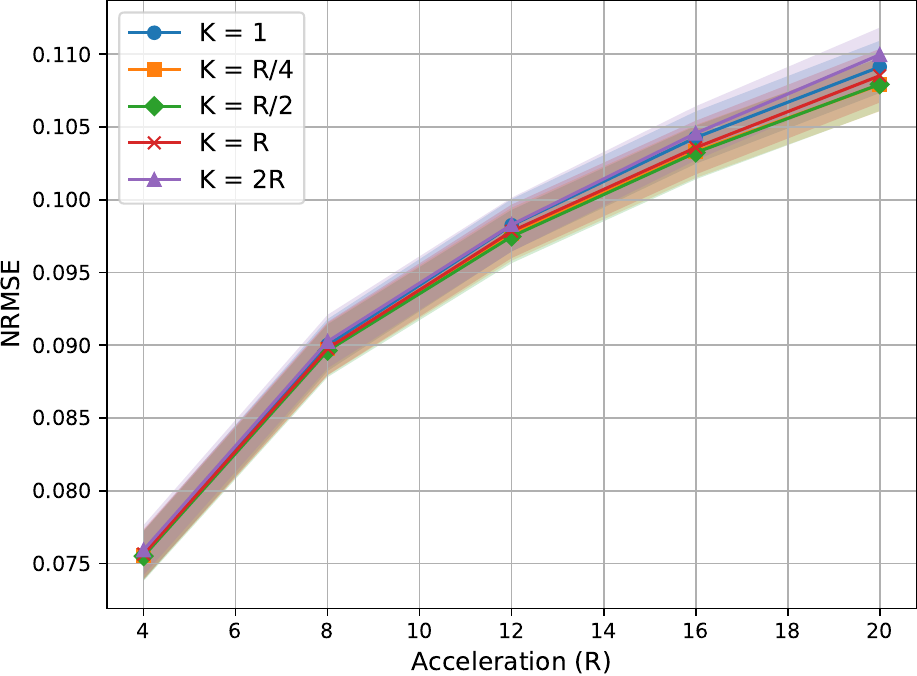}
        \caption{All Knee}
        \label{fig:all_metrics_varyingk}
    \end{subfigure}
    
    \caption{\textbf{Test metrics for our method with a varying top-K parameter.} We compare the mean SSIM \cite{ssim}, PSNR, and NRMSE for reconstructions of slices from the test set using our method. We present results for acceleration factors $R$ in $\{4, 8, 12, 16, 20\}$ and top-K parameter $K$ in $\{1, \frac{R}{4}, \frac{R}{2}, R, 2R\}$ for (a) T2 brain, (b) PD knee, (c) PDFS knee, and (d) combined PD and PDFS knee data. The shaded areas indicate $\pm 1$ standard error.}
    \label{fig:metrics_varyingk}
\end{figure*}

\noindent In this section, we run experiments to analyze the effect of varying the top-K parameter in our method. We train sampling patterns at acceleration factors in $\{4, 8, 12, 16, 20\}$ across all anatomies and contrasts, while varying $K$ within the values $\{1, \frac{R}{4}, \frac{R}{2}, R, 2R\}$. We plot the reconstruction metrics for the test set in Fig.~\ref{fig:metrics_varyingk}. It is clear that the top-K parameter is necessary to regularize the selected k-space points during sampling. Pure greedy selection of sampling points based on the training gradients (i.e., $K=1$) leads to worse performance in each metric in nearly every experiment when compared to larger $K$. The exception to this observation is the experiments on PDFS knee data, which indicate that the highest $K$ values ($K=2R$) lead to worse performance. This is in line with our observations in Sec.~\ref{sec:contrast}, wherein we saw that since PDFS knee data are relatively noisy, patterns learned on this data benefit from sampling low-frequency regions more densely - and thus benefit from a relatively smaller top-K parameter. Overall, the reconstruction quality of each data distribution varies differently with different top-K values, with some (e.g., T2 brain and PD knee) data showing a larger dependence on $K$ and some (e.g., combined knee) showing smaller dependence. 

\section{Posterior Sampling Algorithm}
\label{sec:appendix_algorithm}

We present our posterior sampler, based on DPS \cite{chung2023diffusion}, in Algorithm~\ref{alg:posterior}. We modify the stochastic sampling algorithm from EDM \cite{Karras2022edm} to have no second-order correction and add a log-likelihood step.

\begin{algorithm}[t]
\setstretch{1.1}
\caption{Posterior sampling}
\label{alg:posterior}
\begin{algorithmic}[1]
\Require $\s_{\mathbf{\theta}}(\x_t, t),\: \sigma_{t \in \{t_N, \dots, t_0\}},\: S_{churn}, \rho_{dps}, \y$
\State \textbf{sample} $\x_N \sim \cN(\mathbf{0}, \sigma_{t_N}^2 \I)$ \:\:\:\textcolor{lightgray}{// random starting latent}
\For{$i \in \{N, \dots, 1\}$}
    \State \textbf{sample} $\mathbf{z}_i \sim \cN(\mathbf{0}, \I)$
    \State $\alpha_i \leftarrow min(S_{churn}/N, \sqrt{2}-1)$
    \State $\hat{\sigma}_{t_i} \leftarrow \sigma_{t_i} + \alpha_i \sigma_{t_i}$ \:\:\:\textcolor{lightgray}{// temporarily increase noise level}
    \State $\hat{\x}_i \leftarrow \x_i + \sqrt{\hat{\sigma}_{t_i}^2 - \sigma_{t_i}^2} \mathbf{z}_i$ \:\:\:\textcolor{lightgray}{// inject fresh noise}
    \State $\hat{\x}_0(\hat{\x}_i) \leftarrow \hat{\x}_i + \hat{\sigma}_{t_i}^2 \s_{\mathbf{\theta}}(\hat{\x}_i, \hat{t}_i)$ \:\:\:\textcolor{lightgray}{// Tweedie's formula}
    \State $\hat{\x}_i' \leftarrow \hat{\x}_i + (\hat{\sigma}_{t_i} - \sigma_{t_{i-1}}) \hat{\sigma}_{t_i} \s_{\mathbf{\theta}}(\hat{\x}_i, \hat{t}_i)$ \:\:\:\textcolor{lightgray}{// prior}
    \State $\x_{i-1} \leftarrow \hat{\x}_i' - \rho_{dps} \nabla_{\hat{\x}_i} \norm{\cA(\hat{\x}_0(\hat{\x}_i)) - \y}^2$ \:\:\textcolor{lightgray}{// likelihood}
\EndFor
\State \textbf{return} $\x_0$
\end{algorithmic}
\end{algorithm}

\section{Data Pre-processing Details}
\label{sec:appendix_data}

We pre-processed the raw data by first reducing the FOV in the read out direction by a factor of two and performing noise whitening on each volume using pixels void of anatomical features across all coils. We estimated coil sensitivity maps from the whitened k-space data using ESPIRiT \cite{uecker2014espirit}. Next, we calculated the minimum variance unbiased estimator (MVUE) images using the sensitivity maps and the whitened k-space data. These MVUE images serve as our fully sampled reference during training and testing. Further, we remove the last two slices from each brain volume and the first ten slices from each knee volume since these are highly noisy and lack substantial anatomy. 

We emphasize that at no stage of pre-processing, training, or testing do we crop or pad the data, except for reducing the FOV by a factor of 2 in the readout direction (as we describe earlier in this section). This approach is intended to closely simulate real-world conditions and prevent the artificial introduction or removal of information, which would constitute a ``data crime'' \cite{inversecrimes}. However, we note that our retrospective 3D undersampling is performed on data that were inherently scanned using a 2D multislice sequence.

After reducing the readout FOV by a factor of two, the brain data had spatial dimensions of $320\times320$ while the knee data was $320\times368$. For the knee samples, the raw data provided in the initial dataset was zero-padded from $332$ to $368$, presumably to ensure isotropic pixels in final reconstructions. If left unaddressed, this would lead to incorrect acceleration rates in the sampling patterns. We account for padding in the knee data by enforcing that points in the outermost 18 phase-encode (PE) lines on each side of k-space are not sampled in the patterns, and we leave these lines out when calculating the true acceleration factors.

Prior to training and inference for all methods, we normalize the data since they exhibit a wide range of values from sample to sample. For each slice, we reconstruct a low-resolution image using the central $20\times20$ autocalibration region of the multi-coil k-space data along with the previously estimated sensitivity maps. We take the 99th percentile pixel value of the magnitude of the complex-valued low-resolution image to be our normalization factor and divide the k-space and reference data by this factor. We note that the central $20\times20$ autocalibration region is always fully sampled in our experiments, so our normalization method only relies on information that is available at inference time.  

%% file: references.bib
@inproceedings{multicoil_loupe,
author = {Zhang, Jinwei and Zhang, Hang and Wang, Alan and Zhang, Qihao and Sabuncu, Mert and Spincemaille, Pascal and Nguyen, Thanh D. and Wang, Yi},
title = {Extending LOUPE for K-Space Under-Sampling Pattern Optimization in Multi-coil MRI},
year = {2020},
isbn = {978-3-030-61597-0},
publisher = {Springer-Verlag},
address = {Berlin, Heidelberg},
url = {https://doi.org/10.1007/978-3-030-61598-7_9},
doi = {10.1007/978-3-030-61598-7_9},
booktitle = {Machine Learning for Medical Image Reconstruction: Third International Workshop, MLMIR 2020, Held in Conjunction with MICCAI 2020, Lima, Peru, October 8, 2020, Proceedings},
pages = {91–101},
numpages = {11},
location = {Lima, Peru}
}

@article{aali2024ambient,
  title={Ambient Diffusion Posterior Sampling: Solving Inverse Problems with Diffusion Models trained on Corrupted Data},
  author={Aali, Asad and Daras, Giannis and Levac, Brett and Kumar, Sidharth and Dimakis, Alexandros G and Tamir, Jonathan I},
  journal={arXiv preprint arXiv:2403.08728},
  year={2024}
}

@inproceedings{song2021scorebased,
  title={Score-Based Generative Modeling through Stochastic Differential Equations},
  author={Yang Song and Jascha Sohl-Dickstein and Diederik P Kingma and Abhishek Kumar and Stefano Ermon and Ben Poole},
  booktitle={International Conference on Learning Representations},
  year={2021},
  url={https://openreview.net/forum?id=PxTIG12RRHS}
}

@article{song2019generative,
  title={Generative modeling by estimating gradients of the data distribution},
  author={Song, Yang and Ermon, Stefano},
  journal={Advances in neural information processing systems},
  volume={32},
  year={2019}
}

@article{ho2020denoising,
  title={Denoising diffusion probabilistic models},
  author={Ho, Jonathan and Jain, Ajay and Abbeel, Pieter},
  journal={Advances in Neural Information Processing Systems},
  volume={33},
  pages={6840--6851},
  year={2020}
}

@article{ANDERSON1982313,
title = {Reverse-time diffusion equation models},
journal = {Stochastic Processes and their Applications},
volume = {12},
number = {3},
pages = {313-326},
year = {1982},
issn = {0304-4149},
doi = {https://doi.org/10.1016/0304-4149(82)90051-5},
url = {https://www.sciencedirect.com/science/article/pii/0304414982900515},
author = {Brian D.O. Anderson},
}

@article{vincent2011connection,
  title={A connection between score matching and denoising autoencoders},
  author={Vincent, Pascal},
  journal={Neural computation},
  volume={23},
  number={7},
  pages={1661--1674},
  year={2011},
  publisher={MIT Press}
}

@inproceedings{Karras2022edm,
  author    = {Tero Karras and Miika Aittala and Timo Aila and Samuli Laine},
  title     = {Elucidating the Design Space of Diffusion-Based Generative Models},
  booktitle = {Proc. NeurIPS},
  year      = {2022}
}

@inproceedings{chung2023diffusion,
title={Diffusion Posterior Sampling for General Noisy Inverse Problems},
author={Hyungjin Chung and Jeongsol Kim and Michael Thompson Mccann and Marc Louis Klasky and Jong Chul Ye},
booktitle={The Eleventh International Conference on Learning Representations },
year={2023},
url={https://openreview.net/forum?id=OnD9zGAGT0k}
}

@article{efron2011tweedie,
author = {Efron, Bradley},
year = {2011},
month = {12},
pages = {1602-1614},
title = {Tweedie’s Formula and Selection Bias},
volume = {106},
journal = {Journal of the American Statistical Association},
doi = {10.1198/jasa.2011.tm11181}
}

@article{jaydeblurring,
  title={Deblurring via Stochastic Refinement},
  author={Whang, Jay and Delbracio, Mauricio and Talebi, Hossein and Saharia, Chitwan and Dimakis, Alexandros G and Milanfar, Peyman},
  journal={arXiv preprint arXiv:2112.02475},
  year={2021}
}

@misc{ajilicml21,
      title={Instance-Optimal Compressed Sensing via Posterior Sampling}, 
      author={Ajil Jalal and Sushrut Karmalkar and Alexandros G. Dimakis and Eric Price},
      year={2021},
      eprint={2106.11438},
      archivePrefix={arXiv},
      primaryClass={cs.LG}
}

@article{ddrm,
  title={Denoising Diffusion Restoration Models},
  author={Kawar, Bahjat and Elad, Michael and Ermon, Stefano and Song, Jiaming},
  journal={arXiv preprint arXiv:2201.11793},
  year={2022}
}

@article{cs,
  title={Stable signal recovery from incomplete and inaccurate measurements},
  author={Emmanuel J. Cand{\`e}s and Justin K. Romberg and Terence Tao},
  journal={Communications on Pure and Applied Mathematics},
  year={2005},
  volume={59},
  pages={1207-1223}
}

@ARTICLE{donoho_cs,  
author={Donoho, D.L.},  
journal={IEEE Transactions on Information Theory},   
title={Compressed sensing},
year={2006},  
volume={52},  
number={4},  
pages={1289-1306},  
doi={10.1109/TIT.2006.871582}
}

@inproceedings{phase_retrieval,
 author = {Hand, Paul and Leong, Oscar and Voroninski, Vlad},
 booktitle = {Advances in Neural Information Processing Systems},
 pages = {},
 publisher = {Curran Associates, Inc.},
 title = {Phase Retrieval Under a Generative Prior},
 url = {https://proceedings.neurips.cc/paper/2018/file/1bc2029a8851ad344a8d503930dfd7f7-Paper.pdf},
 volume = {31},
 year = {2018}
}

@misc{1bit,
  doi = {10.48550/ARXIV.2002.01697},
  author = {Liu, Zhaoqiang and Gomes, Selwyn and Tiwari, Avtansh and Scarlett, Jonathan},
  title = {Sample Complexity Bounds for 1-bit Compressive Sensing and Binary Stable Embeddings with Generative Priors},
  publisher = {arXiv},
  year = {2020},
  copyright = {arXiv.org perpetual, non-exclusive license}
}

@article{lustig2007sparse,
  title={Sparse MRI: The application of compressed sensing for rapid MR imaging},
  author={Lustig, Michael and Donoho, David and Pauly, John M},
  journal={Magnetic Resonance in Medicine: An Official Journal of the International Society for Magnetic Resonance in Medicine},
  volume={58},
  number={6},
  pages={1182--1195},
  year={2007},
  publisher={Wiley Online Library}
}

@article{smash,
  title={Simultaneous acquisition of spatial harmonics (SMASH): fast imaging with radiofrequency coil arrays},
  author={Sodickson, Daniel K and Manning, Warren J},
  journal={Magnetic resonance in medicine},
  volume={38},
  number={4},
  pages={591--603},
  year={1997},
  publisher={Wiley Online Library}
}

@article{grappa,
  title={Generalized autocalibrating partially parallel acquisitions (GRAPPA)},
  author={Griswold, Mark A and Jakob, Peter M and Heidemann, Robin M and Nittka, Mathias and Jellus, Vladimir and Wang, Jianmin and Kiefer, Berthold and Haase, Axel},
  journal={Magnetic Resonance in Medicine: An Official Journal of the International Society for Magnetic Resonance in Medicine},
  volume={47},
  number={6},
  pages={1202--1210},
  year={2002},
  publisher={Wiley Online Library}
}

@article{sense,
  title={SENSE: sensitivity encoding for fast MRI},
  author={Pruessmann, Klaas P and Weiger, Markus and Scheidegger, Markus B and Boesiger, Peter},
  journal={Magnetic Resonance in Medicine: An Official Journal of the International Society for Magnetic Resonance in Medicine},
  volume={42},
  number={5},
  pages={952--962},
  year={1999},
  publisher={Wiley Online Library}
}

@ARTICLE{haldar,
  author={Haldar, Justin P. and Setsompop, Kawin},
  journal={IEEE Signal Processing Magazine}, 
  title={Linear Predictability in Magnetic Resonance Imaging Reconstruction: Leveraging Shift-Invariant Fourier Structure for Faster and Better Imaging}, 
  year={2020},
  volume={37},
  number={1},
  pages={69-82},
  doi={10.1109/MSP.2019.2949570}}

@article{UNet_LOUPE,
  title={Deep-learning-based optimization of the under-sampling pattern in MRI},
  author={Bahadir, Cagla D and Wang, Alan Q and Dalca, Adrian V and Sabuncu, Mert R},
  journal={IEEE Transactions on Computational Imaging},
  volume={6},
  pages={1139--1152},
  year={2020},
  publisher={IEEE}
}

@article{sparkling,
author = {Lazarus, Carole and Weiss, Pierre and Chauffert, Nicolas and Mauconduit, Franck and El Gueddari, Loubna and Destrieux, Christophe and Zemmoura, Ilyess and Vignaud, Alexandre and Ciuciu, Philippe},
title = {SPARKLING: variable-density k-space filling curves for accelerated T2*-weighted MRI},
journal = {Magnetic Resonance in Medicine},
volume = {81},
number = {6},
pages = {3643-3661},
doi = {https://doi.org/10.1002/mrm.27678},
eprint = {https://onlinelibrary.wiley.com/doi/pdf/10.1002/mrm.27678},
year = {2019}
}

@article{JMODL_2020,
	doi = {10.1109/jstsp.2020.3004094},
	year = 2020,
	month = {oct},
	publisher = {Institute of Electrical and Electronics Engineers ({IEEE})},
	volume = {14},
	number = {6},
	pages = {1151--1162},
	author = {Hemant Kumar Aggarwal and Mathews Jacob},
	title = {J-{MoDL}: Joint Model-Based Deep Learning for Optimized Sampling and Reconstruction},
	journal = {{IEEE} Journal of Selected Topics in Signal Processing}
}

@article{bjork,
  title={B-Spline Parameterized Joint Optimization of Reconstruction and k-Space Trajectories (BJORK) for Accelerated 2D MRI},
  author={Wang, Guanhua and Luo, Tianrui and Nielsen, Jon-Fredrik and Noll, Douglas C and Fessler, Jeffrey A},
  journal={IEEE Transactions on Medical Imaging},
  volume={41},
  number={9},
  pages={2318--2330},
  year={2022},
  publisher={IEEE}
}

@misc{SNOPY,
      title={Stochastic Optimization of 3D Non-Cartesian Sampling Trajectory (SNOPY)}, 
      author={Guanhua Wang and Jon-Fredrik Nielsen and Jeffrey A. Fessler and Douglas C. Noll},
      year={2022},
      eprint={2209.11030},
      archivePrefix={arXiv},
      primaryClass={eess.SP}
}

@misc{wang2023adaptive,
      title={Adaptive Sampling for Linear Sensing Systems via Langevin Dynamics}, 
      author={Guanhua Wang and Douglas C. Noll and Jeffrey A. Fessler},
      year={2023},
      eprint={2302.13468},
      archivePrefix={arXiv},
      primaryClass={eess.SP}
}

@article{sherry,
  title={Learning the sampling pattern for MRI},
  author={Sherry, Ferdia and Benning, Martin and De los Reyes, Juan Carlos and Graves, Martin J and Maierhofer, Georg and Williams, Guy and Sch{\"o}nlieb, Carola-Bibiane and Ehrhardt, Matthias J},
  journal={IEEE Transactions on Medical Imaging},
  volume={39},
  number={12},
  pages={4310--4321},
  year={2020},
  publisher={IEEE}
}

@article{zibetti,
  title={Fast data-driven learning of parallel MRI sampling patterns for large scale problems},
  author={Zibetti, Marcelo VW and Herman, Gabor T and Regatte, Ravinder R},
  journal={Scientific Reports},
  volume={11},
  number={1},
  pages={19312},
  year={2021},
  publisher={Nature Publishing Group UK London}
}

@ARTICLE{zibetti2,
  author={Zibetti, Marcelo Victor Wust and Knoll, Florian and Regatte, Ravinder R.},
  journal={IEEE Transactions on Computational Imaging}, 
  title={Alternating Learning Approach for Variational Networks and Undersampling Pattern in Parallel MRI Applications}, 
  year={2022},
  volume={8},
  number={},
  pages={449-461},
  doi={10.1109/TCI.2022.3176129}}

@article{Alkan_2024,
   title={AutoSamp: Autoencoding k-space Sampling via Variational Information Maximization for 3D MRI},
   ISSN={1558-254X},
   DOI={10.1109/tmi.2024.3443292},
   journal={IEEE Transactions on Medical Imaging},
   publisher={Institute of Electrical and Electronics Engineers (IEEE)},
   author={Alkan, Cagan and Mardani, Morteza and Liao, Congyu and Li, Zhitao and Vasanawala, Shreyas S. and Pauly, John M.},
   year={2024},
   pages={1–1} }

@article{zbontar2018fastmri,
  title={fastMRI: An open dataset and benchmarks for accelerated MRI},
  author={Zbontar, Jure and Knoll, Florian and Sriram, Anuroop and Murrell, Tullie and Huang, Zhengnan and Muckley, Matthew J and Defazio, Aaron and Stern, Ruben and Johnson, Patricia and Bruno, Mary and others},
  journal={arXiv preprint arXiv:1811.08839},
  year={2018}
}

@article{hammernik,
author = {Hammernik, Kerstin and Klatzer, Teresa and Kobler, Erich and Recht, Michael P. and Sodickson, Daniel K. and Pock, Thomas and Knoll, Florian},
title = {Learning a variational network for reconstruction of accelerated MRI data},
journal = {Magnetic Resonance in Medicine},
volume = {79},
number = {6},
pages = {3055-3071},
doi = {https://doi.org/10.1002/mrm.26977},
url = {https://onlinelibrary.wiley.com/doi/abs/10.1002/mrm.26977},
eprint = {https://onlinelibrary.wiley.com/doi/pdf/10.1002/mrm.26977},
year = {2018}
}

@article{aggarwal2018modl,
  title={MoDL: Model-based deep learning architecture for inverse problems},
  author={Aggarwal, Hemant K and Mani, Merry P and Jacob, Mathews},
  journal={IEEE transactions on medical imaging},
  volume={38},
  number={2},
  pages={394--405},
  year={2018},
  publisher={IEEE}
}

@inproceedings{csgm,
  title={Compressed Sensing using Generative Models},
  author={Ashish Bora and Ajil Jalal and Eric Price and Alexandros G. Dimakis},
  booktitle={ICML},
  year={2017}
}

@inproceedings{robustmri,
 author = {Jalal, Ajil and Arvinte, Marius and Daras, Giannis and Price, Eric and Dimakis, Alexandros G and Tamir, Jon},
 booktitle = {Advances in Neural Information Processing Systems},
 pages = {14938--14954},
 publisher = {Curran Associates, Inc.},
 title = {Robust Compressed Sensing MRI with Deep Generative Priors},
 url = {https://proceedings.neurips.cc/paper/2021/file/7d6044e95a16761171b130dcb476a43e-Paper.pdf},
 volume = {34},
 year = {2021}
}

@article{chulscore,
  title={Score-based diffusion models for accelerated MRI},
  author={Chung, Hyungjin and Ye, Jong Chul},
  journal={Medical Image Analysis},
  volume={80},
  pages={102479},
  year={2022},
  publisher={Elsevier}
}

@article{uecker_diffusionMRI,
  title={Mri reconstruction via data driven markov chain with joint uncertainty estimation},
  author={Luo, Guanxiong and Heide, Martin and Uecker, Martin},
  journal={arXiv preprint arXiv:2202.01479},
  year={2022}
}

@inproceedings{songmri,
title={Solving Inverse Problems in Medical Imaging with Score-Based Generative Models},
author={Yang Song and Liyue Shen and Lei Xing and Stefano Ermon},
booktitle={International Conference on Learning Representations},
year={2022},
url={https://openreview.net/forum?id=vaRCHVj0uGI}
}

@misc{deepjsense,
  doi = {10.48550/ARXIV.2103.02087},
  url = {https://arxiv.org/abs/2103.02087},
  author = {Arvinte, Marius and Vishwanath, Sriram and Tewfik, Ahmed H. and Tamir, Jonathan I.},
  title = {Deep J-Sense: Accelerated MRI Reconstruction via Unrolled Alternating Optimization},
  publisher = {arXiv},
  year = {2021},
  copyright = {Creative Commons Attribution 4.0 International}
}

@article{uecker2014espirit,
  title={ESPIRiT—an eigenvalue approach to autocalibrating parallel MRI: where SENSE meets GRAPPA},
  author={Uecker, Martin and Lai, Peng and Murphy, Mark J and Virtue, Patrick and Elad, Michael and Pauly, John M and Vasanawala, Shreyas S and Lustig, Michael},
  journal={Magnetic resonance in medicine},
  volume={71},
  number={3},
  pages={990--1001},
  year={2014},
  publisher={Wiley Online Library}
}

@ARTICLE{ssim,
  author={Zhou Wang and Bovik, A.C. and Sheikh, H.R. and Simoncelli, E.P.},
  journal={IEEE Transactions on Image Processing}, 
  title={Image quality assessment: from error visibility to structural similarity}, 
  year={2004},
  volume={13},
  number={4},
  pages={600-612},
  doi={10.1109/TIP.2003.819861}}

@article{imaging_survey,
  title={Deep learning techniques for inverse problems in imaging},
  author={Ongie, Gregory and Jalal, Ajil and Metzler, Christopher A and Baraniuk, Richard G and Dimakis, Alexandros G and Willett, Rebecca},
  journal={IEEE Journal on Selected Areas in Information Theory},
  volume={1},
  number={1},
  pages={39--56},
  year={2020},
  publisher={IEEE}
}

@article{deepdecoder,
  title={Deep Decoder: Concise Image Representations from Untrained Non-convolutional Networks},
  author={Reinhard Heckel and Paul Hand},
  journal={ArXiv},
  year={2019},
  volume={abs/1810.03982}
}

@inproceedings{kingma2015adam,
  author       = {Diederik P. Kingma and
                  Jimmy Ba},
  title        = {Adam: {A} Method for Stochastic Optimization},
  booktitle    = {3rd International Conference on Learning Representations, {ICLR} 2015,
                  San Diego, CA, USA, May 7-9, 2015, Conference Track Proceedings},
  year         = {2015},
  url          = {http://arxiv.org/abs/1412.6980},
  bibsource    = {dblp computer science bibliography, https://dblp.org}
}

@article{dip,
  title={Deep Image Prior},
  author={Dmitry Ulyanov and Andrea Vedaldi and Victor S. Lempitsky},
  journal={2018 IEEE/CVF Conference on Computer Vision and Pattern Recognition},
  year={2018},
  pages={9446-9454}
}

@inproceedings{ambientgan,
title={Ambient{GAN}: Generative models from lossy measurements},
author={Ashish Bora and Eric Price and Alexandros G. Dimakis},
booktitle={International Conference on Learning Representations},
year={2018},
url={https://openreview.net/forum?id=Hy7fDog0b},
}

@article{
inversecrimes,
author = {Efrat Shimron  and Jonathan I. Tamir  and Ke Wang  and Michael Lustig },
title = {Implicit data crimes: Machine learning bias arising from misuse of public data},
journal = {Proceedings of the National Academy of Sciences},
volume = {119},
number = {13},
pages = {e2117203119},
year = {2022},
doi = {10.1073/pnas.2117203119},
URL = {https://www.pnas.org/doi/abs/10.1073/pnas.2117203119},
eprint = {https://www.pnas.org/doi/pdf/10.1073/pnas.2117203119},}

@inproceedings{ong2019sigpy,
  title={SigPy: a python package for high performance iterative reconstruction},
  author={Ong, Frank and Lustig, Michael},
  booktitle={Proceedings of the ISMRM 27th Annual Meeting, Montreal, Quebec, Canada},
  volume={4819},
  year={2019}
}

@article{Heckel2024DeepLF,
  title={Deep Learning for Accelerated and Robust MRI Reconstruction: a Review},
  author={Reinhard Heckel and Mathews Jacob and Akshay Chaudhari and Or Perlman and Efrat Shimron},
  journal={ArXiv},
  year={2024},
  volume={abs/2404.15692},
  url={https://api.semanticscholar.org/CorpusID:269330329}
}

@INPROCEEDINGS{wei2022,
  author={Peng, Wei and Feng, Li and Zhao, Guoying and Liu, Fang},
  booktitle={2022 IEEE/CVF Conference on Computer Vision and Pattern Recognition (CVPR)}, 
  title={Learning Optimal K-space Acquisition and Reconstruction using Physics-Informed Neural Networks}, 
  year={2022},
  volume={},
  number={},
  pages={20762-20771},
  doi={10.1109/CVPR52688.2022.02013}}

@article{weiss2021pilot,
  title={{PILOT}: Physics-Informed Learned Optimized Trajectories for Accelerated {MRI}},
  author={Weiss, Tomer and Senouf, Ortal and Vedula, Sanketh and Michailovich, Oleg and Zibulevsky, Michael and Bronstein, Alex and others},
  journal={Machine Learning for Biomedical Imaging},
  year={2021}
}

@article{wright2014non,
  title={Non-Cartesian parallel imaging reconstruction},
  author={Wright, Katherine L and Hamilton, Jesse I and Griswold, Mark A and Gulani, Vikas and Seiberlich, Nicole},
  journal={Journal of Magnetic Resonance Imaging},
  volume={40},
  number={5},
  pages={1022--1040},
  year={2014},
  publisher={Wiley Online Library}
}

@inproceedings{blau2018perception,
  title     = {The Perception-Distortion Tradeoff},
  author    = {Blau, Yochai and Michaeli, Tomer},
  booktitle = {Proceedings of the IEEE Conference on Computer Vision and Pattern Recognition (CVPR)},
  pages     = {6228--6237},
  year      = {2018}
}

@article{wang2025welldesigned,
  title         = {Well-Designed k-Space Coverage is Important for Good MRI Denoising},
  author        = {Wang, Jiayang and Haldar, Justin P.},
  year          = {2025},
    journal={arXiv preprint arXiv:2511.05735},
}

@phdthesis{sriramThesis,
  title  = {Solving Inverse Problems with Deep Learning: From Untrained to Pre-trained Models},
  author = {Ravula, Sriram},
  school = {The University of Texas at Austin},
  year   = {2024},
}
